%% file: Paper.tex
\documentclass{article}
\PassOptionsToPackage{numbers}{natbib}
\usepackage[preprint]{neurips_2019}
\usepackage[utf8]{inputenc} 
\usepackage[T1]{fontenc}    
\usepackage{hyperref}       
\usepackage{url}            
\usepackage{booktabs}       
\usepackage{nicefrac}       
\usepackage{microtype}      
\usepackage{graphicx}

\usepackage{enumerate}
\usepackage{bbm, bm}

\graphicspath{{./figures/}}
\usepackage{amsfonts, amsthm, amssymb}
\usepackage{mathtools}
\usepackage{float}
\usepackage{wrapfig}
\usepackage{dsfont}
\usepackage{algorithm}
\usepackage[noend]{algorithmic}
\usepackage{dsfont}
\usepackage{environ}
\usepackage{thmtools, thm-restate}
\usepackage{subcaption}
\usepackage[x11names]{xcolor}

\DeclareMathOperator*{\argmax}{arg\,max\,}

\newcommand{\elimparameter}{\rho}
\newcommand{\ucbparameter}{\gamma}
\newcommand{\maxgapelim}{{\tt MaxGapElim }}
\newcommand{\maxgapucb}{{\tt MaxGapUCB }}
\newcommand{\maxgapbandit}{MaxGap-bandit }
\newcommand{\maxgaptwoucb}{{\tt MaxGapTop2UCB }}

\newcommand{\gapucb}{{\tt U \Delta}}
\newcommand{\gaplcb}{{\tt L \Delta}}
\newcommand{\gaprucb}{{\tt U \Delta}^r}
\newcommand{\gaplucb}{{\tt U \Delta}^l}

\mathchardef\mhyphen="2D
\newtheorem{theorem}{Theorem}
\newtheorem{lemma}{Lemma}
\newtheorem{corollary}{Corollary}
\newtheorem{theorem*}{Theorem}
\newtheorem{remark}{Remark}
\usepackage[capitalize]{cleveref}
\crefname{lemma}{Lemma}{Lemmas}
\Crefname{lemma}{Lemma}{Lemmas}

\usepackage[backgroundcolor=White,textwidth=0.75in,disable]{todonotes}
\newcommand{\todos}[2][]{\todo[color=Chocolate1,size=\tiny,#1]{S: #2}} 

\title{MaxGap Bandit: Adaptive Algorithms for Approximate Ranking}
\usepackage{times}
\author{%
 Sumeet Katariya \thanks{Authors contributed equally and are listed
 alphabetically.}\\
 University of Wisconsin\\
 Madison, WI 53706.\\
 \texttt{sumeetsk@gmail.com}
 \And
 Ardhendu Tripathy \textsuperscript{\textasteriskcentered}\\
 University of Wisconsin\\
 Madison, WI 53706.\\
 \texttt{astripathy@wisc.edu}
 \And
 Robert Nowak\\
 University of Wisconsin\\
 Madison, WI 53706.\\
 \texttt{rdnowak@wisc.edu}
}

\begin{document}

\maketitle
\vspace{-5pt}
\begin{abstract}%
This paper studies the problem of adaptively sampling from $K$ distributions
(arms) in order to identify the largest gap between any two adjacent means. We
call this the MaxGap-bandit problem. This problem arises naturally in
approximate ranking, noisy sorting, outlier detection, and top-arm
identification in bandits.  The key novelty of the MaxGap bandit problem is that
it aims to adaptively determine the natural partitioning of the distributions
into a subset with larger means and a subset with smaller means, where the split
is determined by the largest gap rather than a pre-specified rank or threshold.
Estimating an arm’s gap requires sampling its neighboring arms in addition to
itself, and this dependence results in a novel hardness parameter that
characterizes the sample complexity of the problem. We propose elimination and
UCB-style algorithms and show that they are minimax optimal. Our experiments
show that the UCB-style algorithms require $6 \mhyphen 8$x fewer samples than non-adaptive sampling to achieve the same error.
\end{abstract}

\input{Introduction}

\input{Algorithm}

\input{Prelim}

\input{Analysis}

\input{Lowerbound}
\input{Experiments}

\input{Conclusion}

\clearpage
\bibliographystyle{plainnat}
\bibliography{References}
\input{Appendix}
\end{document}

%% file: Introduction.tex
\vspace{-5pt}
\section{Introduction}
\label{sec:introduction}
Consider an algorithm that can draw i.i.d.\ samples from $K$ unknown
distributions.  The goal is to partially rank the distributions according to
their (unknown) means.  This model encompasses many problems including best-arms
identification  in multi-armed bandits, noisy sorting and ranking, and outlier
detection.  Partial ranking is often preferred to complete ranking because correctly ordering distributions
with nearly equal means is an expensive task  (in terms of number of required
samples).  Moreover, in many applications it is arguably unnecessary to resolve
the order of such close distributions.  This observation motivates algorithms that aim
to recover a partial ordering into groups/clusters of distributions with similar
means.  This entails identifying large ``gaps'' in the ordered sequence of
means.  The focus of this paper is the fundamental problem of finding the {\em
largest gap} by sampling adaptively. Identification of the largest gap separates the distributions into two groups, and thus recursive application would allow one to identify any number of groupings in a partial order.   

\begin{figure}
    \centering
    \includegraphics[width=\textwidth]{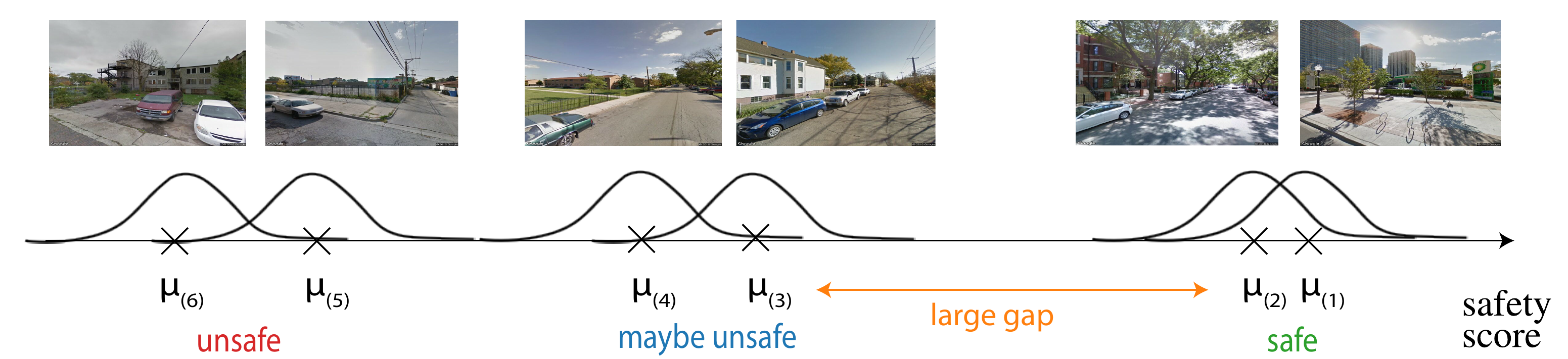}
    \caption{Six representative images from Chicago streetview dataset and their safety (Borda) scores.}
    \label{fig:street-made-up}
\end{figure}
As illustration, consider a subset of images from the Chicago streetview dataset \citep{coarserankingdataset} shown in \cref{fig:street-made-up}. In
this study, people were asked to judge how safe each scene looks
\citep{katariya2018adaptive}, and a larger mean indicates a safer looking scene. 
While each person has a different sense of how safe an image looks, when
aggregated there are clear trends in the safety scores (denoted by $\mu_{(i)}$) of the images. 
\cref{fig:street-made-up} schematically shows the distribution of scores given
by people as a bell curve below each image. Assuming the sample means are close
to their true means, one can nominally classify them as `safe', `maybe unsafe' and
`unsafe' as indicated in \cref{fig:street-made-up}. Here we have implicitly used
the large gaps $\mu_{(2)}-\mu_{(3)}$ and $\mu_{(4)}-\mu_{(5)}$ to mark the
boundaries. Note that finding the safest image (best-arm identification) is hard as we need a
lot of human responses to decide the larger mean between the two rightmost
distributions; it is also arguably unnecessary. A common way to address this
problem is to specify a tolerance $\epsilon$
\citep{even2006action}, and stop sampling if the means are less
than $\epsilon$ apart; however determining this can require $\Omega(1/\epsilon^2)$ samples. Distinguishing the
top $2$ distributions from the rest is easy and can be efficiently done using
top-$m$ arm identification \citep{kalyanakrishnan2010efficient}, however this
requires the experimenter to prescribe the location $m=2$ where a large gap exists
which is unknown. {\em Automatically identifying natural splits in the set of
distributions is the aim of the new theory and algorithms we propose.
We call this problem of adaptive sampling to find the largest gap the
\maxgapbandit problem.} 

\subsection{Notation and Problem Statement}
\label{sec:notation}
We will use multi-armed bandit terminology and notation throughout the paper.
The $K$ distributions will be called {\em arms} and drawing a
sample from a distribution will be refered to as {\em sampling the arm}. 
Let $\mu_i \in \mathbb{R}$ denote the mean of the $i$-th arm, $i \in
\{1,2,\ldots, K\} =: [K]$. We add a parenthesis around the subscript $j$ to
indicate the $j$-th largest mean, i.e., $\mu_{(K)} \leq \mu_{(K-1)} \leq \cdots
\leq \mu_{(1)}$. For the $i$-th arm, we define its gap $\Delta_i$ to be the maximum of
its left and right gaps, i.e.,
\begin{equation}
\Delta_i \ = \ \max\{\mu_{(\ell)}-\mu_{(\ell+1)} \, , \, \mu_{(\ell-1)}-\mu_{(\ell)}\} \quad
\text{ where }\mu_i = \mu_{(\ell)}.
\label{eq:gapi}
\end{equation}
We define $\mu_{(0)} = -\infty$ and $\mu_{(K+1)} = \infty$ to
account for the fact that extreme arms have only one gap. 
The goal of the \maxgapbandit problem is to (adaptively) sample the
arms and return two clusters
\begin{equation*}
C_1 = \{(1), (2), \ldots, (m)\} \quad \text{ and }\quad C_2 = \{(m+1), \ldots,
(K)\},
\end{equation*}
where $m$ is the rank of the arm with the largest gap between {\em adjacent}
means, i.e.,
\begin{equation}\label{eq:location of largest gap}
m \ = \ \argmax_{j\in [K-1]} \mu_{(j)}-\mu_{(j+1)}.
\end{equation}
The mean values are unknown as is the ordering of the arms according to their means. 
A solution to the \maxgapbandit problem is
an algorithm which given a probability of error $\delta>0$,
samples the arms and upon stopping partitions $[K]$ into two clusters
$\widehat{C}_1$ and $\widehat{C}_2$ such that
\begin{align}\mathbb{P}(\widehat{C}_1 \neq C_1) \le \delta. 
\label{eq:clustering guarantee}
\end{align}
This setting is known as the fixed-confidence setting
\citep{garivier2016optimal}, and the goal is to achieve the probably
correct  clustering using as few samples as possible. 
In 
the sequel, we assume that $m$ is uniquely defined and let $\Delta_{\max} = \Delta_{i^\ast}$ where $\mu_{i^\ast} = \mu_{(m)}$. 

\subsection{Comparison to a Naive Algorithm: Sort then search for MaxGap}
\label{sec:simple algorithm}
The \maxgapbandit problem is not equivalent to best-arm identification
on $\binom{K}{2}$ gaps since the \maxgapbandit problem requires identifying the largest
gap between {\em adjacent} arm means (best arm identification on $\binom{K}{2}$ gaps would always identify $\mu_{(1)}-\mu_{(K)}$ as the
largest gap). This suggests a naive two-step algorithm: we first sample the arms enough number of times so as 
to identify all pairs of adjacent arms (i.e., we sort the arms according to
their means), and then run a best-arm identification
bandit algorithm \citep{jamieson2014lil}
on the $(K-1)$ gaps between adjacent arms to identify the largest gap (an unbiased
sample of the gap can be obtained by taking the difference
of the samples of the two arms forming the gap).

\begin{wrapfigure}[2]{l}{.45\textwidth}
    \centering
    \vspace{-10pt}
    \includegraphics[width=\linewidth]{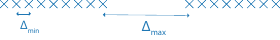}
    \vspace{-17pt}
    \caption{Configuration with one large gap}
    \label{fig:toy_problem}
\end{wrapfigure}
We analyze the sample complexity of this naive algorithm in \cref{sec:appendix
simple algorithm}
, and discuss the results here for an example
configuration. Consider the arrangement of means shown in \cref{fig:toy_problem} where there is one large gap $\Delta_{\max}$ and all the
other gaps are equal to $\Delta_{\min} \ll \Delta_{\max}$. The naive algorithm
has a sample complexity $\Omega(K/\Delta_{\min}^2)$ (the first sorting step
requires these many samples) which can be very large. 
Is this sorting of the arm means necessary? For instance, we do not 
need to sort $K$ real numbers in order to cluster them according to the largest gap.\footnote{First find the smallest and largest numbers, say $a$ and $b$ respectively. 
Divide the interval $[a,b]$
into $K+1$ equal-width bins and map each number to its corresponding bin, while
maintaining the smallest and largest number in each bin. Since at least one bin
is empty by the pigeonhole principle, the largest gap is between two numbers
belonging to different bins. Calculate the gaps between the bins and report the
largest as the answer.} The algorithms we propose in this paper solve the
\maxgapbandit problem without necessarily sorting the arm means. For the
configuration in \cref{fig:toy_problem} they require 
$\tilde{O}(K/\Delta_{\max}^2)$ samples, giving a saving of approximately
$(\Delta_{\max}/\Delta_{\min})^2$ samples. 

The analysis of our algorithms suggests a novel hardness parameter for the
\maxgapbandit problem that we discuss next. We let $\Delta_{i,j}
\mathrel{\mathop:}= \mu_j-\mu_i$ for all $i,j \in [K]$. We show in
\cref{sec:analysis} that the number of samples taken from distribution $i$ due to its right gap is inversely proportional to the square of 
\begin{equation}\label{eq:gamma-R}
\ucbparameter_i^r  \ := \ \max_{j:\Delta_{i,j}>0} \min\big\{\Delta_{i,j} \, , \, \Delta_{\rm
max}-\Delta_{i,j} \big\} \ .
\end{equation}
For the left gap of $i$ we define $\gamma_i^l$ analogously. The total number of samples drawn from 
distribution $i$ is inversely proportional to the square of
$\gamma_i:=\min\{\gamma_i^r, \gamma_i^l\}$. 
The intuition for \cref{eq:gamma-R} is that distribution $i$ can be eliminated
quickly if there is another distribution $j$ that has a moderately large gap
from $i$ (so that this gap can be quickly detected), but not too large (so that the
gap is easy to distinguish from $\Delta_{\rm max}$), and \eqref{eq:gamma-R} chooses
the best $j$. We discuss \eqref{eq:gamma-R} in detail 
in \cref{sec:analysis}, where we show that our algorithms use $\widetilde
O\big(\sum_{i\in [K]/\{(m),(m+1)\} } \gamma_i^{-2} \log( K/\delta \gamma_i)\big)$
samples to find the largest gap with probability at least $1-\delta$. 
This sample complexity is minimax optimal. 

\subsection{Summary of Main Results and Paper Organization}
In addition to motivating and formulating the \maxgapbandit problem, we make the following contributions. First, we design elimination and UCB-style
algorithms as solutions to the \maxgapbandit problem that do not require
sorting the arm means (\cref{sec:algorithm}). These algorithms require computing upper bounds on the gaps
$\Delta_i$, which can be formulated as a mixed integer optimization problem. We 
design a computationally efficient dynamic programming subroutine to solve this
optimization problem and this is our second contribution (\cref{sec:prelim gap upper bound}).
Third, we analyze the sample complexity of our proposed algorithms, and discover
a novel problem-hardness parameter (\cref{sec:analysis}). This parameter arises because of the arm interactions in the \maxgapbandit
problem where, in order to reduce uncertainty in the 
value of an arm's gap, 
we not only need to sample the said arm
but also its neighboring arms. Fourth, we show that this sample complexity is
minimax optimal (\cref{sec:lowerbound}). Finally, we evaluate the empirical
performance of our algorithms on simulated and real datasets and observe that
they require $6\mhyphen 8$x fewer samples than non-adaptive sampling to achieve
the same error (\cref{sec:experiments}).

\section{Related Work}
\label{sec:related work}
One line of related research is best-arm identification in multi-armed bandits.
A typical goal in this setting is to identify the top-$m$ arms with largest
means, where $m$ is a {\em prespecified} number
\citep{kalyanakrishnan2010efficient,kalyanakrishnan2012pac,agarwal2017learning,
BubeckWV_top-m,gabillon2012best,chen2017nearly,jun2016top, even2006action,
mannor2004sample}. As explained in \cref{sec:introduction}, our motivation
behind formulating the \maxgapbandit
problem is to have an adaptive algorithm which finds the
``natural'' set of top arms as delineated by the largest gap in consecutive
mean values. Our work can also be used to automatically detect ``outlier'' arms \citep{zhuang2017identifying}.

The \maxgapbandit problem is different from the standard multi-armed bandit
because of the local dependence of an arm's gap on other arms. Other
best-arm settings where an arm's reward can inform the quality of other arms
include linear bandits \citep{soare2014best} and
combinatorial bandits
\citep{chen2014combinatorial,huang2018combinatorial}. In these problems, the
decision space is known to the learner, i.e., the vectors corresponding
to the arms in linear bandits and the subsets of arms over which the objective
function is to be optimized in combinatorial bandits is known to the learner. However in our problem, we do not know the sorted order of the arm means, i.e., the set of all valid gaps is unknown \textit{a priori}. 
Our problem does not reduce to these settings.

Another related problem is noisy sorting and ranking.   Here the typical goal is to sort a list using noisy  pairwise comparisons. Our framework encompasses noisy  ranking based on Borda scores \citep{agarwal2017learning}.  The Borda score of an item is the probability that it is ranked higher in a pairwise comparison with another item chosen uniformly at random.  In our setting, the Borda score is the mean of each distribution.  Much of the theoretical computer science  literature on this topic assumes a bounded noise model for comparisons (i.e., comparisons are probably correct with a  positive margin) \citep{feige1994computing,davidson2014top,braverman2016parallel,mao2018minimax}. This is unrealistic in many real-world applications since near equals or outright ties are not uncommon.  The largest gap problem we study  can be used to (partially) order items into two natural groups, one with large means and one with small means.  
Previous related work considered a similar problem with prescribed (non-adaptive) quantile groupings \citep{katariya2018adaptive}.

%% file: Algorithm.tex
\vspace{-5pt}
\section{MaxGap Bandit Algorithms}
\label{sec:algorithm}
\vspace{-5pt}
We propose elimination \citep{even2006action} and UCB \citep{jamieson2014lil}
style algorithms for the \maxgapbandit problem. These algorithms operate on the
arm \emph{gaps} instead of the arm \emph{means}. 
The subroutine to construct confidence intervals on the gaps
(denoted by $\gapucb_a(t)$) using 
confidence intervals on the arm means (denoted by $[l_a(t), r_a(t)]$) is described in \cref{alg:gapucb} in
\cref{sec:prelim gap upper bound}, and this subroutine 
is used by all three algorithms described in this section. 
\subsection{Elimination Algorithm: $\maxgapelim$}
\label{sec:elimination algorithm}
At each time step, $\maxgapelim$ (\cref{alg:maxgapelim}) samples all arms in an
active set consisting of arms $a$ whose
gap upper bound $\gapucb_a$ is larger than the global lower bound $\gaplcb$ on the maximum gap, and stops when there are only two arms in the active set.
\begin{algorithm}[htb]
    \caption{$\maxgapelim$}
    \label{alg:maxgapelim}
    \begin{algorithmic}[1]
        \STATE Initialize active set $A = [K]$
        \FOR[// rounds]{$t=1,2,\dots$}
            \STATE $\forall \,a\in A$, sample arm $a$, compute $[l_a(t)$,
            $r_a(t)]$ using \eqref{eq:mean upper lower bound}. \hfill
            \textit{//arm confidence intervals}
            \STATE $\forall \,a\in A$, compute $\gapucb_a(t)$ using
            \cref{alg:gapucb}.\hfill \textit{// upper bound on arm max gap}
            \STATE Compute $\gaplcb(t)$ using \eqref{eq:max gap lower bound}. \hfill \textit{// lower bound on max gap}
            \STATE $\forall \,a\in A$, if $\gapucb_a(t) \le \gaplcb(t), A = A
            \setminus a$. \hfill \textit{// Elimination}
            \STATE If $|A|=2$, stop. Return clusters using max gap in
            the empirical means.
            \hfill \textit{// Stopping condition} 
        \ENDFOR
    \end{algorithmic}
\end{algorithm}
\vspace{-10pt}
\subsection{UCB algorithms: $\maxgapucb$ and $\maxgaptwoucb$}
\label{sec:ucb algorithm}
\begin{algorithm}[tbh]
    \caption{$\maxgapucb$}
    \label{alg:maxgapucb}
    \begin{algorithmic}[1]
        \STATE Initialize $\mathcal{U} = [K]$.
        \FOR{$t=1,2,\dots$}
        \STATE $\forall a\in \mathcal{U}$, sample $a$ and update $[l_a(t), r_a(t)]$ using \eqref{eq:mean upper lower bound}. 
        \STATE $\forall a\in [K]$, compute $\gapucb_a(t)$ using
        \cref{alg:gapucb}.
        \STATE Let $M_1(t) = \max_{j \in [K]} \gapucb_j(t)$. Set $\mathcal{U} =
        \{a: \gapucb_a(t) = M_1(t)\}$. \hfill  \textit{// highest gap-UCB arms}
        \STATE If $\exists\, i,j$ such that $T_i(t)+T_j(t) \ge c \sum_{a \notin
        \{i,j\}} T_a(t)$, stop. \hfill \textit{// stopping condition}
        \ENDFOR
    \end{algorithmic}
\end{algorithm}
\begin{algorithm}[h]
    \caption{$\maxgaptwoucb$}
    \label{alg:maxgaptwoucb}
    \begin{algorithmic}[1]
        \STATE Initialize $\mathcal{U}_1 \cup \mathcal{U}_2 = [K]$.
        \FOR{$t=1,2,\dots$}
        \STATE $\forall a\in \mathcal{U}_1 \cup \mathcal{U}_2$, sample $a$ and update $[l_a(t), r_a(t)]$ using \eqref{eq:mean upper lower bound}. 
        \STATE $\forall a\in [K]$, compute $\gapucb_a(t)$ using
        \cref{alg:gapucb}. 
        \STATE Let $M_1(t) = \max_{j \in [K]} \gapucb_j(t)$. Set $\mathcal{U}_1 =
        \{a: \gapucb_a(t) = M_1(t)\}$. \hfill \textit{// highest gap-UCB arms}
        \STATE Let $M_2(t) = \max_{j \in [K]\setminus \mathcal{U}_1} \gapucb_j(t)$.
        Set $\mathcal{U}_2 = \{a: \gapucb_a(t) = M_2(t)\}$. \hfill \textit{// $2$nd highest gap-UCB}
        \STATE Compute $\gaplcb(t)$ using \eqref{eq:max gap lower bound}. If $M_2(t) < \gaplcb(t)$, stop.
        \ENDFOR
    \end{algorithmic}
\end{algorithm}
$\maxgapucb$ (\cref{alg:maxgapucb}) is motivated from the principle of ``optimism in the face
of uncertainty''. 
It samples {\em
  all} arms with the highest gap upper bound. Note that there are at
least two arms with the highest gap upper bound because any gap is
shared by at least two arms (one on the right and one on the left). 
The stopping condition is akin to the stopping
condition in \citet{jamieson2014lil}. 

Alternatively, we can use an LUCB
\citep{kalyanakrishnan2012pac}-type algorithm that samples arms which have
the two highest gap upper bounds, and stops when the second-largest
gap upper bound is smaller than the global lower bound $\gaplcb(t)$ . We refer to this 
algorithm as $\maxgaptwoucb$ 
(\cref{alg:maxgaptwoucb}).


%% file: Prelim.tex
\vspace{-10pt}
\section{Confidence Bounds for Gaps}
\label{sec:prelim gap upper bound}
In this section we explain how to construct confidence bounds for the arm gaps
(denoted by $\gapucb_a$ and $\gaplcb$) using confidence bounds for the arm means (denoted
by $[l_a,r_a]$). These bounds are key ingredients 
for the algorithms described in
\cref{sec:algorithm}. 

\begin{algorithm}[tb]
    \caption{Procedure to find $\gapucb_a(t)$}
    \label{alg:gapucb}
    \begin{algorithmic}[1]
        \STATE Set $P_a^r = \{i: l_i(t) \in [l_a(t), r_a(t)]\}$.
        \STATE $\gaprucb_a(t) = \max\limits_{i \in P_a^r}
        \,\left\{G_a^r(l_i(t),t)\right\}$, where
        $G_a^r(x,t)$ is given by \eqref{eq:G^R}. \hfill \textit{// eqn.
    \eqref{eq:right_gap_UB}}
        \STATE Set $P_a^l = \left\{i: r_i(t) \in [l_a(t), r_a(t)] \right\}$.
        \STATE $\gaplucb_a(t) = \max\limits_{i \in P_a^l}\,
        \left\{G_a^l(r_j(t),t)\right\}$, where $G_a^l(x,t)$ is given by
        \eqref{eq:G^L}. \hfill \textit{// eqn. \eqref{eq:left_gap_UB}}
        \RETURN $\gapucb_a(t) \leftarrow \max\{\gaprucb_a(t), \gaplucb_a(t)\}$
    \end{algorithmic}
\end{algorithm}
Given i.i.d.\ samples from arm $a$, an empirical mean
$\widehat\mu_a$ and confidence interval on the arm mean can be constructed using
standard methods.  Let $T_a(t)$ denote the number of samples from arm
$a$ after $t$ time steps of the algorithm.  Throughout our analysis
and experimentation we use confidence intervals on the mean of the form 
\begin{align}
    l_a(t) = \hat{\mu}_a(t) - c_{T_a(t)} \text{ and }
    r_a(t) = \hat{\mu}_a(t) + c_{T_a(t)}, \text{ where }
    c_s = \sqrt{\tfrac{\log(4Ks^2/\delta)}{s}}.
    \label{eq:mean upper lower bound}
\end{align}
The confidence intervals are chosen so that \citep{jamieson2014best}
\begin{equation}
    \mathbb{P}(\forall\,t \in \mathbb{N}, \forall\,a \in [K], \mu_a \in [l_a(t),r_a(t)] ) \ge 1-\delta.
    \label{eq:valid confidence interval}
\end{equation}

Conceptually, the confidence intervals on the arm means can be used to
construct upper
confidence bounds on the mean gaps $\{\Delta_i\}_{i\in [K]}$ in the
following manner.
Consider all possible configurations of the arm means that satisfy the confidence interval constraints in \eqref{eq:mean upper lower bound}. 
Each configuration fixes the gaps associated with any arm $a \in [K]$. 
Then the maximum gap value over all configurations is the upper confidence bound on arm $a$'s gap; we denote it as $\gapucb_a$. 
The above procedure can be formulated as a mixed integer linear program (see
\cref{sec:appendix mixed integer program}).
In the algorithms in \cref{sec:algorithm}, this optimization problem needs to be
solved at every time $t$ and for every arm
$a \in [K]$ before querying a new sample, which can be practically infeasible. In
\cref{alg:gapucb}, we give an efficient $O(K^2)$ time dynamic programming
algorithm to compute $\gapucb_a$. We next explain the main ideas used in this
algorithm, and refer the reader to \cref{sec:appendix gapucb algorithm proof}
for the proofs.

Each arm $a$ has a right and left gap, $\Delta_a^r:=\mu_{(\ell-1)}-\mu_{(\ell)}$
and $\Delta_a^l :=\mu_{(\ell)}-\mu_{(\ell+1)}$, where $\ell$ is the rank of $a$, i.e., $\mu_a =
\mu_{(\ell)}$. We
construct separate upper bounds $\gaprucb_a(t)$ and $\gaplucb_a(t)$ for
these gaps and then define $\gapucb_a(t)= \max\{\gaprucb_a(t), \gaplucb_a(t)\}$.
Here we provide an intuitive description for how the bounds are computed,
focusing on $\gaprucb_a(t)$ as an example. 

\begin{figure}[tb]
    \centering
    \includegraphics[width=.4\textwidth]{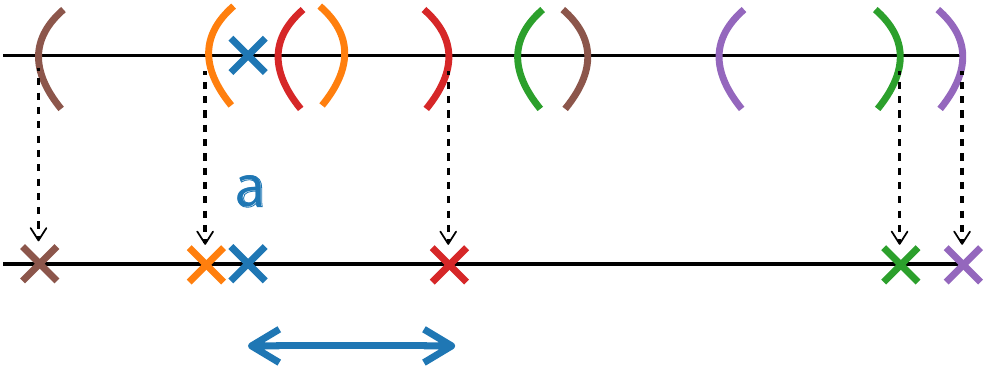}\hfill
    \includegraphics[width=.4\textwidth]{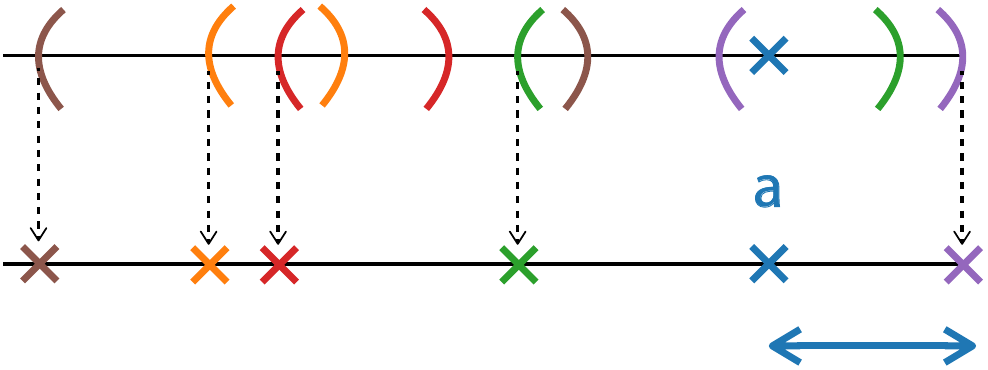}\\
    (a) \hspace{270pt} (b)
    \caption{Computing maximum right gap of blue arm when its true
      mean is known (at position indicated by blue x), while
        the other means are known only to lie within their confidence intervals. (a) If
        there exist arms that cannot go to the left of blue (red, green, purple), the largest right gap for blue
is obtained by placing all arms that can go to the left of blue at their
left boundaries and the remaining arms at their rightmost
positions. (b) If all arms can go to the left of blue, the largest right
gap for blue is obtained by placing the arm with the largest right confidence
bound (purple) at its
right boundary and all other arms at their left boundaries.}
\vspace{-10pt}
\label{fig:rightgapucbfixedpos}
\end{figure}

To start, suppose the true
mean of arm $a$ is known exactly, while the means of other arms are
only known to lie within their confidence intervals. If there exist arms
that cannot go to the left of arm $a$, one can see that the
largest right gap for $a$ is obtained by placing all arms that can go
to the left of $a$ at their leftmost positions, and all remaining arms
at their rightmost positions, as shown in
\cref{fig:rightgapucbfixedpos}(a). If however all arms can go
to the left of arm $a$, the configuration that gives the largest right
gap for $a$ is obtained by placing the arm with the 
largest upper bound 
at its right boundary, and all other arms at their left
boundaries, as illustrated in \cref{fig:rightgapucbfixedpos}(b). We
define a function $G^r_a(x,t)$ that takes as input a known position
$x$ for the mean of arm $a$ and the confidence intervals of all other
arms at time $t$, and returns the maximum right gap for arm $a$ using the above
idea as follows.
\begin{equation}\label{eq:G^R}
G_a^r(x, t) =
\begin{cases}
\min_{j: l_j(t) > x} r_j(t) - x &\text{if } \{j: l_j(t) > x\} \neq \emptyset,\\
\max_{j \neq a} r_j(t) - x &\text{otherwise}.
\end{cases}
\end{equation}
However, the true mean of arm $a$ is not known exactly but only that
it lies within its confidence interval. The insight that helps here is
  that $G_a^r(x, t)$ must achieve its maximum when $x$ is at 
  one of the finite locations in $\{l_j(t): l_a(t) \le l_j(t) \le r_a(t)\}$. We define
  $P_a^r:=\{j:l_a(t) \le l_j(t) \le r_a(t)\}$ as the set of arms pertinent to
  compute the right gap upper bound of $a$, and then the maximum possible right gap
  of $a$ is 
\begin{align}
    \vspace{-5pt}
\gaprucb_a(t) & = \max\{G_a^r(l_j(t), t): j \in P_a^r\}. \label{eq:right_gap_UB}
\end{align}
An upper bound for the left gap $\gaplucb_a$ can be similarly obtained. 
We explain this and give a proof of correctness in \cref{sec:appendix gapucb
algorithm proof}.

 The algorithms also use a single global lower bound on the maximum gap. To do
 so, we sort the items according to their empirical means, and find
 partitions of items that are clearly separated in terms of their
 confidence intervals. At time $t$, let $(i)_t$ denote the arm with
 the $i^{\text{th}}$-largest empirical mean, i.e.,
$\widehat\mu_{(K)_t}(t) \le \dots \widehat\mu_{(2)_t}(t) \le
\widehat\mu_{(1)_t}(t)$ (this can be different from the true ranking
which is denoted by $(\cdot)$ without the subscript $t$).
We {\em detect} a nonzero gap at arm $k$ if
$\max_{a \in \{(k+1)_t,
\dots,(K)_t \} }
r_a(t) < \min_{a \in \{(1)_t, \dots, (k)_t\}} l_a(t)$. 
Thus, a lower bound on the largest gap is 
\begin{equation}
    \gaplcb(t) = \max_{k \in [K-1]} \left( \min_{a \in \{(1)_t, \dots, (k)_t\}} l_a(t)
    -\hspace{-5pt} \max_{a \in \{(k+1)_t, \dots,(K)_t \}} r_a(t) \right).
    \label{eq:max gap lower bound}
\end{equation}

%% file: Analysis.tex
\section{Analysis}
\label{sec:analysis}
In this section, we first state the accuracy and sample complexity guarantees for $\maxgapelim$ and
$\maxgapucb$, and then discuss our results. The proofs can be
found in the Supplementary material. 

\begin{restatable}{theorem}{TheoremAccuracy}
    With probability $1-\delta$, $\maxgapelim$, $\maxgapucb$ and $\maxgaptwoucb$ cluster the arms according to the
    maximum gap, i.e., they satisfy \eqref{eq:clustering guarantee}.
    \label{thm:accuracy}
\end{restatable}
The number of times arm $a$ is sampled by both the algorithms depends on a
parameter $\ucbparameter_a = \min \{\ucbparameter_a^l, \ucbparameter_a^r\}$ where 
\begin{align}
    \ucbparameter_a^r &= \max_{j: 0 < \Delta_{a,j} < \Delta_{\max}} \min
    \{\Delta_{a,j}, (\Delta_{\max} - \Delta_{a,j})\}\label{eq:gammair
    definition}  \\
    \ucbparameter_a^l &= \max_{j: 0 < \Delta_{j,a} < \Delta_{\max}} \min
    \{\Delta_{j,a}, (\Delta_{\max} - \Delta_{j,a})\}, \label{eq:gammail
    definition}.
\end{align}
The maxima is assumed to be $\infty$ in \eqref{eq:gammair definition} and
\eqref{eq:gammail definition} if there is no $j$ that satisfies the
constraint to account for edge arms. The quantity $\ucbparameter_a$ acts as a measure of hardness for
arm $a$; \cref{thm:sample complexity} states that $\maxgapelim$ and $\maxgapucb$ sample arm $a$
at most $\tilde{O}(1/\ucbparameter_a^2)$ number of times (up to log factors). 
\begin{theorem}
    With probability $1-\delta$, the sample complexity of $\maxgapelim$ and
    $\maxgapucb$ is bounded by $$O\left( \sum_{a \in [K]\setminus\{(m),(m+1)\}}
    \frac{\log (K/\delta \ucbparameter_a)}{\ucbparameter_a^2} \right)$$
    \label{thm:sample complexity}
\end{theorem}
Next, we provide intuition for why the sample complexity depends on the
parameters in $\eqref{eq:gammair
definition}$ and \eqref{eq:gammail definition}. In particular, we show that 
$O( (\ucbparameter_a^r)^{-2})$ (resp.\ $O( (\ucbparameter_a^l)^{-2})$) is the number of samples of $a$ required to
rule out arm $a$'s right (resp.\ left) gap from being the largest gap. 

\begin{wrapfigure}[8]{l}{.5\textwidth}
    \centering
    \vspace{-10pt}
    \includegraphics[width=\linewidth]{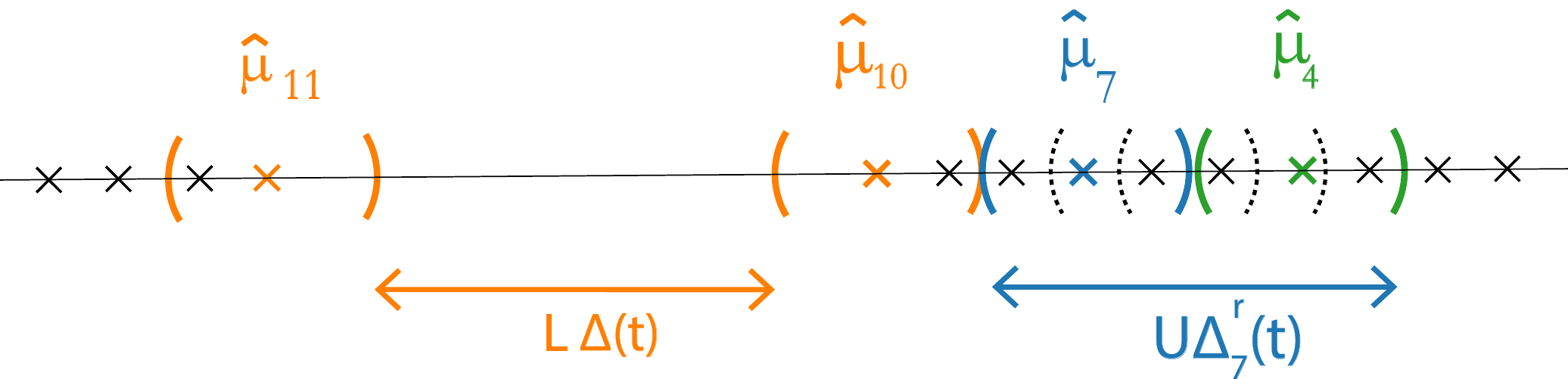}
    \caption{Arm $a=7$ is eliminated when a helper arm $j=4$ is found.}
    \label{fig:discussion}
\end{wrapfigure}
Let us focus on the right gap for simplicity. 
To understand how (\ref{eq:gammair definition}) naturally arises, consider  
\cref{fig:discussion}, which denotes the confidence intervals on the means at
some time $t$. A lower bound on the gap $\gaplcb(t)$ can be computed
between the left and right confidence bounds of arms $10$ and $11$ respectively
as shown. Consider the computation of the upper bound $\gaprucb_7(t)$ on the
right gap of arm $a = 7$. Arm $4$
lies to the right of arm $7$ with high
probability (unlike the arms with dashed confidence intervals), so the upper bound $\gaprucb_7(t) \leq r_4(t)-l_7(t)$.
Considering only the right gap for simplicity, as
soon as $\gaprucb_7(t) < \gaplcb(t)$, arm $7$ can be eliminated as a candidate
for the maximum gap.

Thus, an arm $a$ is removed from consideration
as soon as we find a \emph{helper} arm $j$ (arm $4$ in \cref{fig:discussion}) that satisfies two properties: (1) the
confidence interval of arm $j$ is disjoint from that of arm $a$, and (2) the
upper bound $\gaprucb_a(t) = r_j(t) - l_a(t) < \gaplcb(t)$. The first of these
conditions gives rise to the term $\Delta_{a,j}$ in \eqref{eq:gammair
definition}, and the second condition gives rise to the term
$(\Delta_{\max}-\Delta_{a,j})$. Since any arm $j$ that satisfies these
conditions can act as a helper for arm $a$, we take the maximum over all arms $j$ to yield the smallest sample complexity for arm $a$.

This also shows that if all arms are either very close to $a$ or at a distance
approximately $\Delta_{\max}$ from $a$, then the upper bound $\gaprucb_7(t)
= r_4(t)-l_7(t) > \gaplcb(t)$ and arm $7$ cannot be
eliminated. Thus arm $a$ could have a small gap with respect to its adjacent
arms, but if there is a large gap in the vicinity of arm $a$, it cannot be
eliminated quickly. This illustrates that the maximum gap identification problem
is not equivalent to best-arm identification on gaps. \cref{sec:lowerbound}
formalizes this intuition.

%% file: Lowerbound.tex
\section{Minimax Lower Bound}
\label{sec:lowerbound}
\begin{wrapfigure}[10]{lrio}{.45\textwidth}
    \vspace{-10pt}
\centering
\includegraphics[width=\linewidth]{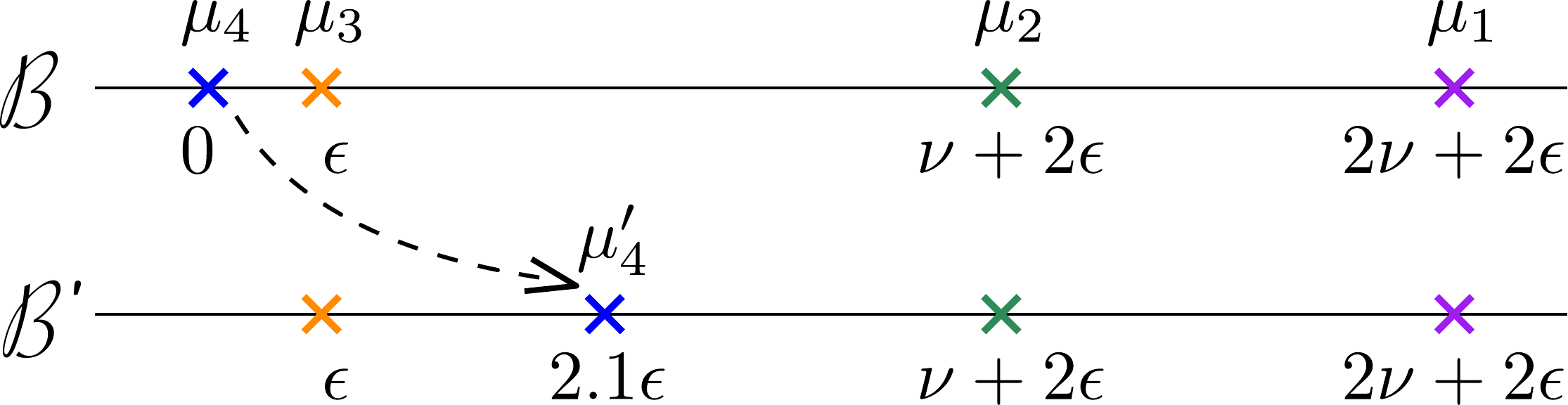} 
\caption{Changing the original bandit model $\mathcal{B}$ to
  $\mathcal{B}'$. $\mu_4$ is
  shifted to the right by $2.1\epsilon$. As a result, the maximum gap
  in $\mathcal{B}'$ is between green and purple.  }
\label{fig:change_of_measure}
\end{wrapfigure}
\vspace{-10pt}
In this section, we demonstrate that the MaxGap problem is fundamentally different from best-arm identification on
gaps. 
We construct a problem instance and prove a lower bound on the
number of samples needed by any probably correct algorithm. The lower bound
matches the upper bounds in the previous section for this instance.  
\begin{restatable}{lemma}{Minimax}
Consider a model $\mathcal{B}$ with $K = 4$ normal distributions
$\mathcal{P}_i = \mathcal{N}(\mu_i, 1)$, where 
\begin{equation*}
\mu_4 = 0,\quad \mu_3 = \epsilon,\quad \mu_2 = \nu + 2\epsilon,\quad \mu_1 = 2\nu + 2\epsilon,
\end{equation*}
for some $\nu \gg \epsilon > 0$. Then any algorithm that is correct
with probability at least $1-\delta$ must collect
$\Omega(1/\epsilon^2)$ samples of arm $4$ in expectation. 
\label{lem:lower bound lemma}
\end{restatable}
\vspace{-5pt}
\emph{Proof Outline:} The proof uses a standard change of measure argument
\citep{garivier2016optimal}. We construct another problem instance
$\mathcal{B}'$ which has a different maximum gap clustering compared to
$\mathcal{B}$ (see \cref{fig:change_of_measure}; the maxgap clustering in
$\mathcal{B}$ is $\{4,3\} \cup \{2,1\}$, while the maxgap clustering in
$\mathcal{B}'$ is $\{4,3,2\}\cup \{1\}$), and show that in order to distinguish between $\mathcal{B}$ and
$\mathcal{B}'$, any probably correct algorithm must collect at least
$\Omega(1/\epsilon^2)$ samples of arm $4$ in expectation (see \cref{sec:appendix lower
bound} for details). 

From the definition of $\gamma_a$ using \eqref{eq:gammair definition},\eqref{eq:gammail definition}, it is easy to check that $\gamma_4 =
\epsilon$. Therefore, for problem
instance $\mathcal{B}$ our algorithms find the maxgap clustering using at most $O(\log(\epsilon/\delta)/\epsilon^2)$
samples of arm $4$ (\textit{c.f.\ }\cref{thm:sample complexity}). This essentially
matches the lower bound above.  

This example illustrates why the maximum gap identification
problem is different from a simple best-arm identification on gaps.
Suppose an oracle told a best-arm identification algorithm the ordering of the arm means. 
Using the ordering it can convert the $4$-arm maximum gap problem
$\mathcal{B}$ to a best arm 
identification problem on $3$ \emph{gaps}, with distributions
$\mathcal{P}_{4,3} = \mathcal{N}(\epsilon, 2), 
\mathcal{P}_{3,2} = \mathcal{N}(\nu + \epsilon, 2)$, and
$\mathcal{P}_{2,1} = \mathcal{N}(\nu, 2)$.
The best-arm algorithm can sample arms $i$ and $i+1$ to get a sample of the gap
$(i+1,i)$. We know from standard best-arm identification analysis
\citep{jamieson2014lil} that the gap $(4,3)$ can be eliminated from being the
largest by sampling it (and hence arm $4$) $O(1/\nu^2)$ times, which can be
arbitrarily lower than the $1/\epsilon^2$ lower bound in \cref{lem:lower bound lemma}. Thus the ordering information given to the best-arm identification algorithm is crucial for it to quickly 
identify the larger gaps. The problem we solve in this paper is identifying the maximum gap when 
the ordering information is \emph{not} available. 

%% file: Experiments.tex
\vspace{-10pt}
\section{Experiments}
\label{sec:experiments}
\vspace{-5pt}
\subsection{Streetview Dataset}
\label{sec:streetview experiment}
\vspace{-10pt}
\begin{wrapfigure}[16]{l}{.4\textwidth}
    \centering
    \vspace{-10pt}
    \includegraphics[width=.9\linewidth]{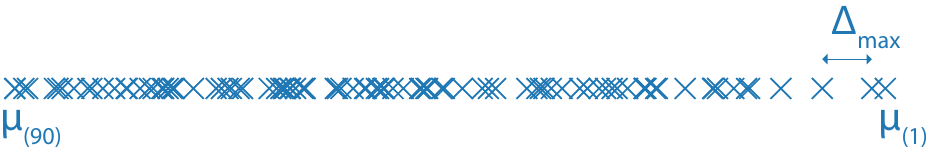}
    
    \vspace{-15 pt}
    \hspace{0.4\textwidth} (a)
    
    \includegraphics[width=0.85\linewidth]{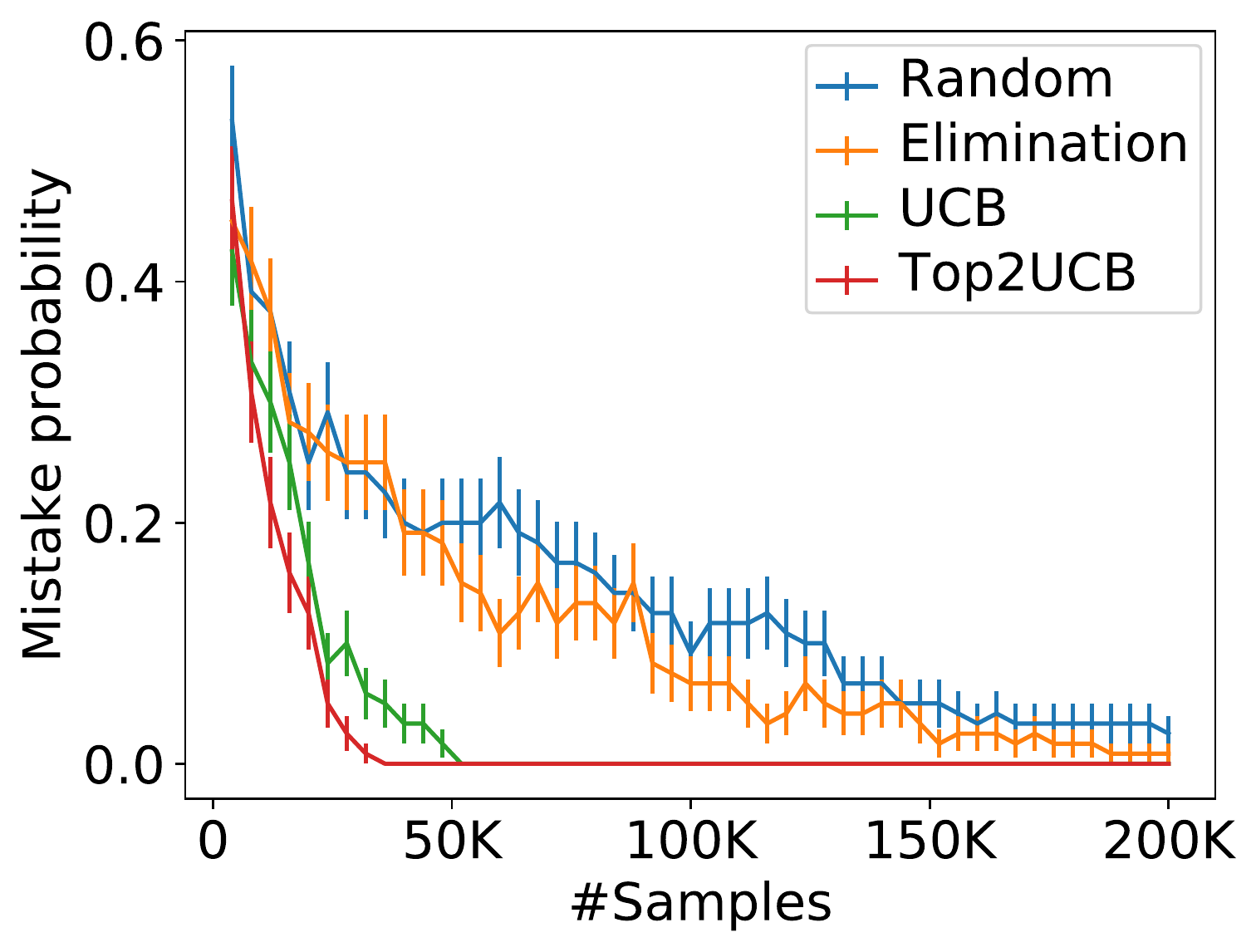} 
    
     \vspace{-10 pt}
    \hspace{0.4\textwidth} (b)
     \vspace{-5 pt}
    \caption{(a) Borda safety scores for Streetview images. (b) Probability of returning a wrong cluster.}
    \label{fig:borda_performance_expsection}
\end{wrapfigure}
In our first experiment we study performance on the Streetview dataset \citep{coarserankingdataset,katariya2018adaptive}
whose means are plotted in 
\cref{fig:borda_performance_expsection}(a)
. We have $K=90$
arms, where each arm is a normal distribution with mean equal to the Borda safety score of the image and standard
deviation $\sigma = 0.05$. The largest gap of $0.029$ is between arms $2$ and
$3$, and the second largest gap
is $0.024$. 
In \cref{fig:borda_performance_expsection}(b), we plot the fraction of times $\hat{C}_1 \ne
    \{1,2\}$ in $120$ runs as a function of the number
of samples, for four algorithms, viz., random (non-adaptive) sampling, $\maxgapelim$, $\maxgapucb$, and
$\maxgaptwoucb$. The error bars denote standard deviation over the runs. 
$\maxgapucb$ and $\maxgaptwoucb$ require $6\mhyphen
7$x fewer samples than random sampling. 
\vspace{-5pt}
\subsection{Simulated Data}
\label{sec:simulated experiment}
\begin{figure}[t]
    \centering
    \begin{minipage}[c][4.6cm][c]{0.34\textwidth}
    \centering
    \includegraphics[width=1.1\linewidth]{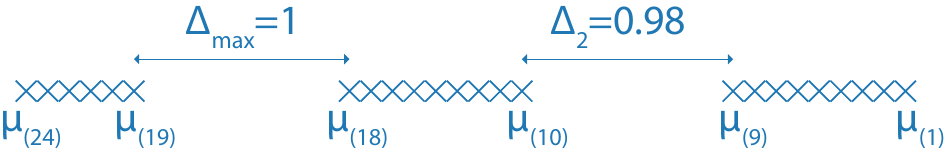}        
        
        \hspace{0.15\textwidth} (a)
        
        (b)\includegraphics[width=\textwidth]{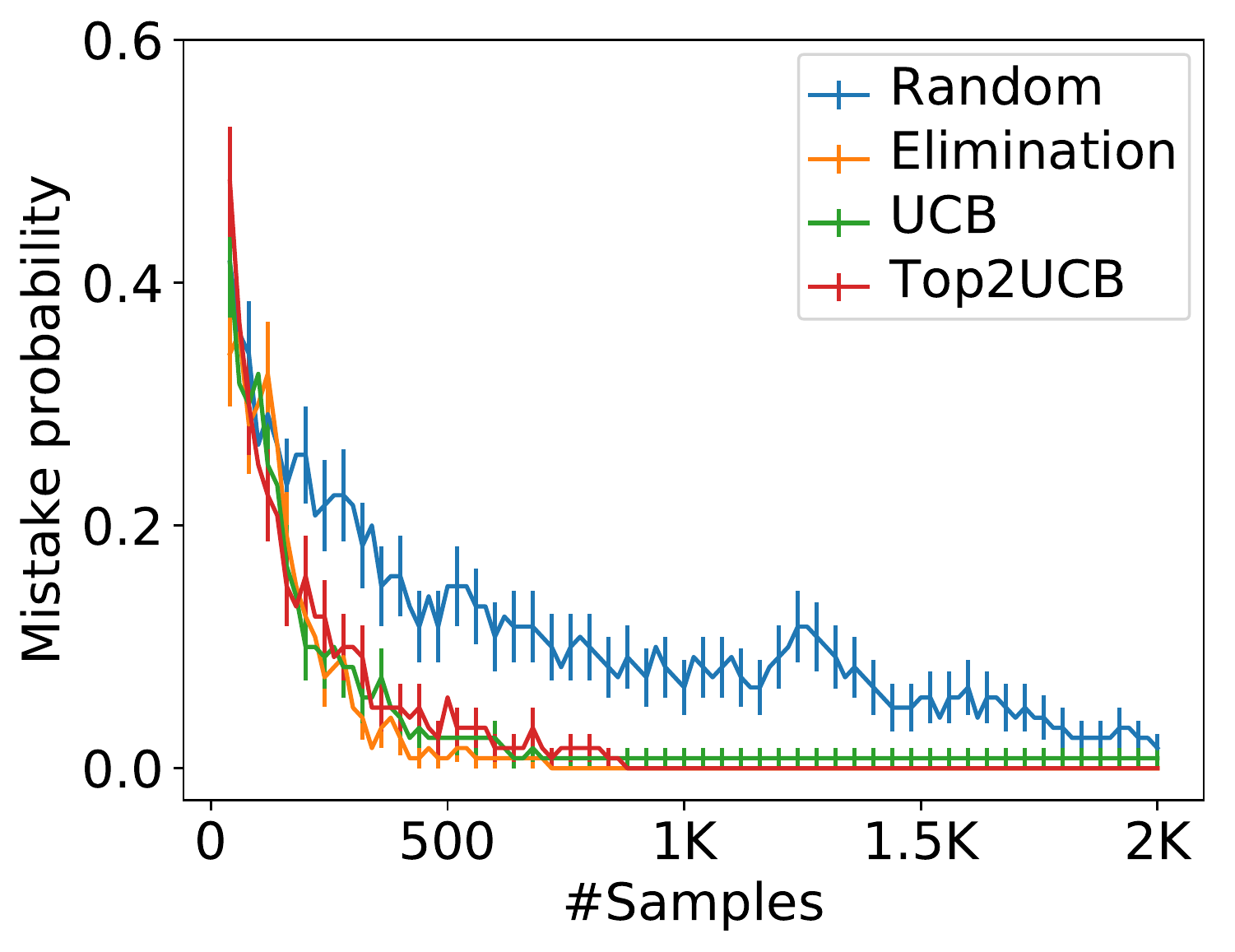}
        \vspace{-10pt}
    \end{minipage}
    \begin{minipage}[c][4.6cm][c]{0.64\textwidth}
    \centering
        \includegraphics[scale=0.39]
                {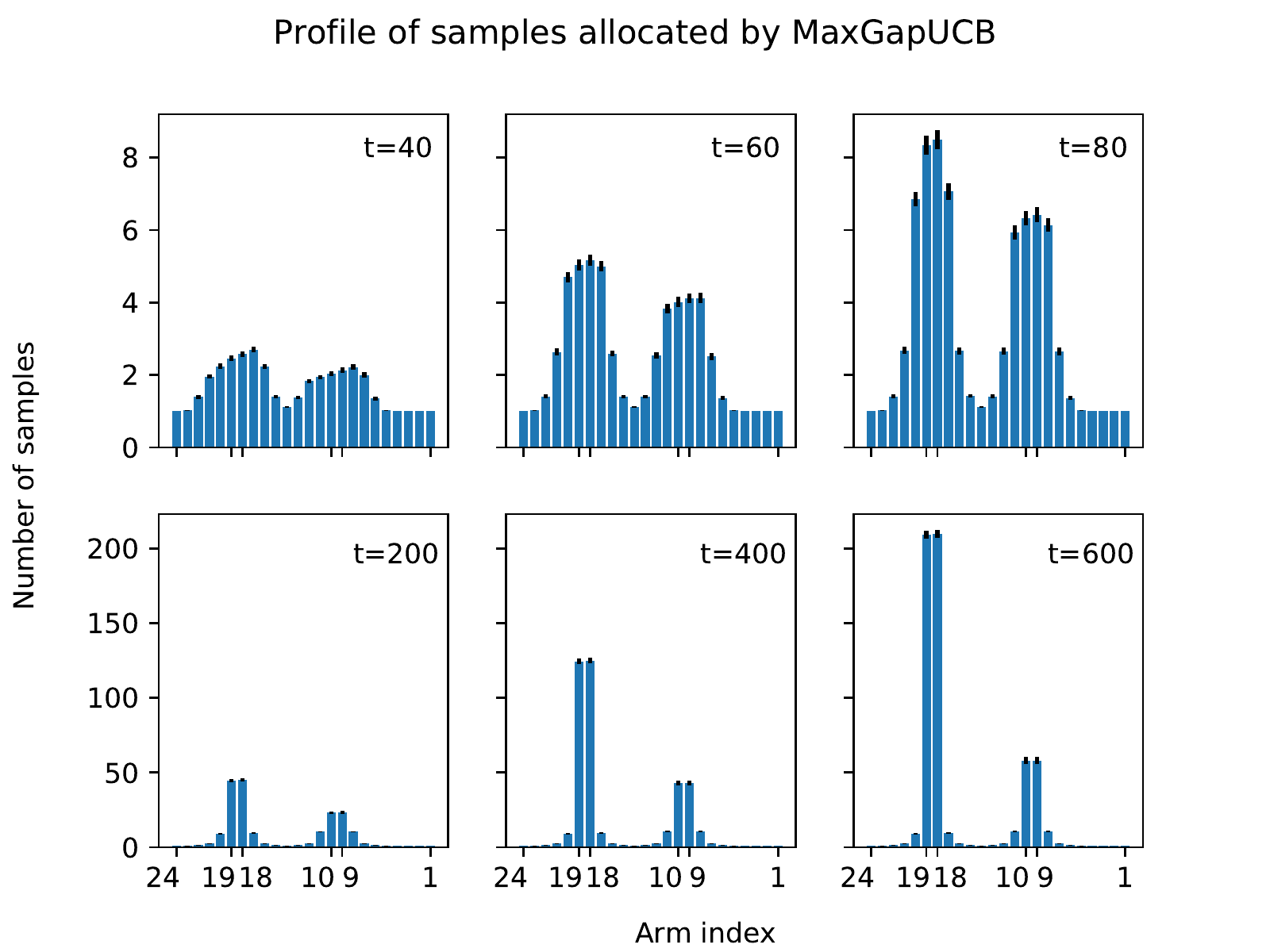}
        
        \vspace{-5 pt}
        \hspace{0.07\textwidth} (c)
    \end{minipage}
    \caption{(a) Two large gaps. (b) Clustering error
probability for means shown in \cref{fig:simulated experiment}(a). (c) The profile of samples allocated by \maxgapucb to each arm in (a) at different time steps.}
\vspace{-10pt}
    \label{fig:simulated experiment}
\end{figure}

In the second experiment, we study the performance on on a simulated set of means containing two large gaps. 
The mean distribution plotted in \cref{fig:simulated
experiment}(a) has $K=24$ arms ($\mathcal{N}(\cdot,1)$), with two large
mean gaps $\Delta_{10,9}=0.98, \Delta_{19,18}=1.0$, and remaining small gaps
($\Delta_{i+1,i}=0.2$ for $i\notin \{9,18\}$). We expect to see a big advantage
for adaptive sampling in this example because almost every sub-optimal arm has a
\emph{helper} arm (see \cref{sec:analysis}) which can help eliminate
it quickly, and adaptive algorithms can then focus on
distinguishing the two large gaps. A non-adaptive algorithm on the other hand
would continue sampling all arms. 
We plot the fraction of times $C_1 \ne \{1,\dots,18\}$ in $120$ runs in \cref{fig:simulated experiment}(b), and
see that the active algorithms identify the largest gap in 8x fewer samples. 
To visualize the adaptive allocation of samples to the arms, we plot in
\cref{fig:simulated experiment}(c) the number
of samples queried for each arm at different time steps by {\tt MaxGapUCB}.
Initially, \maxgapucb allocates samples uniformly over all the arms. After a few
time steps, we see a bi-modal profile in the number of samples. Since all arms
that achieve the largest $\gapucb$ are sampled, we see that several arms that
are near the pairs $(10, 9)$ and $(19, 18)$ are also sampled frequently. As time
progresses,
only the pairs $(10,9)$ and $(19,18)$ get sampled, and eventually more samples
are allocated to the larger gap $(19,18)$ among the two.
%

%% file: Conclusion.tex
\vspace{-10pt}
\section{Conclusion}
\label{sec:conclusion}
\vspace{-10pt}
%
%
In this paper, we proposed the \maxgapbandit problem: a novel maximum-gap identification problem 
that can be used as a basic primitive for clustering and approximate ranking.
Our analysis shows a novel hardness parameter for the problem, and our
experiments show 6-8x gains compared to non-adaptive algorithms. 
We use simple Hoeffding based confidence intervals in our analysis for
simplicity, but better bounds can be obtained using tighter confidence intervals
\citep{jamieson2014lil}.

%% file: Appendix.tex
\clearpage
\appendix

\section{Details for \cref{sec:simple algorithm}: Comparison to a Naive Algorithm}
\label{sec:appendix simple algorithm}
The naive algorithm first sorts the arms to determine the adjacent arms for
every arm, and then runs a best-arm identification bandit algorithm on the gaps
to identify the largest gap. An unbiased sample of the gap between two arms can
be obtained by taking the difference of the samples from the two arms. Here we
analyze the sample complexity of the naive algorithm for a general 
arrangement of the means.

Consider an arm $i \notin \{(m), (m+1)\}$, and let us analyze the number of
times arm $i$ is sampled by 
the naive algorithm. Let $\Delta_{i,j} =\mu_j-\mu_i$.
Then $\Delta_{i}^r = \min_{j:\Delta_{i,j}>0} \Delta_{i,j}$ is the right gap
of arm $i$ ($\Delta_i^l$ is defined analogously). In the first step of the naive algorithm, arm $i$ needs to be sampled at
least $(\Delta_i^r)^{-2}$ times to determine its right neighbor. Once the right
neighbor has
been determined, the best-arm identification requires at 
least $(\Delta_{\max}-\Delta_{i}^r)^{-2}$ samples to distinguish arm $i$'s right
gap from $\Delta_{\max}$. Since samples from the first step can be reused,
the minimum number of samples required by the naive algorithm to rule out arm
$i$'s right gap is $(\tilde{\gamma}_i^{r})^{-2}$ where
\begin{equation}
    \tilde{\ucbparameter}_i^r = \min_{j:\Delta_{i,j}>0} \{\Delta_{i,j},
\Delta_{\max}-\Delta_{i,j}\}
\label{eq:sample complexity naive algorithm}
\end{equation}
We can define $(\tilde{\gamma}_i^l)^{-2}$ analogously, and the naive algorithm
collects $\Omega(\tilde{\gamma}_i^{-2})$ from arm $i$, where $\tilde{\gamma}_i =
\min\{\tilde{\gamma}_i^r, \tilde{\gamma}_i^l\}$.

The hardness parameter of our active algorithms that is analogous to \eqref{eq:sample complexity naive
algorithm} is given by \eqref{eq:gamma-R}, repeated here for
convenience
\begin{equation}
\ucbparameter_i^r  \ := \ \max_{j:\Delta_{i,j}>0} \min\big\{\Delta_{i,j} \, , \, \Delta_{\rm
max}-\Delta_{i,j} \big\} \ .
\label{eq:gamma-R repeated}
\end{equation}
Comparing \eqref{eq:gamma-R repeated} to \eqref{eq:sample complexity naive
algorithm}, we see that $\gamma_i^r > \tilde{\gamma}_i^r$.

For the toy problem discussed in \cref{sec:simple algorithm}, if we assume that
$\Delta_{\min}< \Delta_{\max}/2$, we have that $\tilde{\gamma}_i =
\Delta_{\min}$, while $\gamma_i =
\Delta_{\max}/2\,\forall\,i\notin \{(m),(m+1)\}$, which results in $(\Delta_{\max}/\Delta_{\min})^2$ order
savings in the number of samples.

\section{Details for \cref{sec:prelim gap upper bound}: Confidence Bounds for Gaps}
\label{sec:gapbounds}
We first explain the mixed integer program formulation for obtaining the upper confidence
bounds on the mean gaps in \cref{sec:appendix mixed integer program}, and then
prove the validity of \cref{alg:gapucb} in \cref{sec:appendix gapucb algorithm
proof}.

\subsection{MIP Formulation of Confidence Bounds for Gaps}
\label{sec:appendix mixed integer program}
Conceptually, the confidence intervals on the arm means can be used to
construct upper
confidence bounds on the mean gaps $\{\Delta_i\}_{i\in [K]}$ in the
following manner.
Consider all possible configurations of the arm means that satisfy the confidence interval constraints in \eqref{eq:mean upper lower bound}. 
Each configuration fixes the gaps associated with any arm $a \in [K]$. 
Then the maximum gap value over all configurations is the upper confidence bound
on arm $a$'s gap; we denote it as $\gapucb_a$. 

If we focus on the right gap of
arm $a$, the above procedure is equivalent to solving the following optimization problem.
\begin{align}
    \gaprucb_a(t) \triangleq &\max\limits_{b \in [K] \setminus \{a\}}
    \max\limits_{\mu_1',\dots,\mu_K'} \mu_b' - \mu_a' \label{eq:realize}\\
    &\text{ subject to: } l_i(t) \le \mu_i' \le r_i(t) \quad \forall i \in [K], \text{ and} \label{eq:realize_constraints1}\\
    &\hphantom{\text{ subject to: }} \mu_i' \notin (\mu_a', \mu_b') \,\forall\,
    i \in [K] \setminus \{a,b\}.\label{eq:realize_constraints2}
\end{align}
Constraint \eqref{eq:realize_constraints1} ensures
that $\mu_i'$ is in the confidence interval for the mean of arm $i$ at time $t$,
and constraint \eqref{eq:realize_constraints2} ensures that arm $b$ is the right
neighbor of arm $a$.

The constraint \eqref{eq:realize_constraints2} is a sorting constraint that can only be formulated using
a binary variable. For an arm $i \in [K] \setminus \{a,b\}$ \eqref{eq:realize_constraints2} can be
formulated using a constant $M$ as 
\begin{subequations}
\begin{align}
    \mu_i' &\le \mu_a' + M(1-z_i), \label{eq:realize_constraints3a}\\
    \mu_i' &\ge \mu_b' - Mz_i, \label{eq:realize_constraints3b}\\
    z_i &\in \{0,1\}. \label{eq:realize_constraints3c}
\end{align}
\end{subequations}
The value of $M$ is chosen to be large number. Replacing constraint \eqref{eq:realize_constraints2} 
by constraints \eqref{eq:realize_constraints3a}, \eqref{eq:realize_constraints3b}, \eqref{eq:realize_constraints3c} 
for all $i \in [K]\setminus \{a,b\}$ gives an equivalent optimization problem 
whose optimum value is $\gaprucb_a(t)$. This can be seen to be true by 
considering the cases based on the value of $z_i$. If $z_i = 0, \mu_i' \geq \mu_b'$ 
and if $z_i = 1, \mu_i' \leq \mu_a'$. Because $M$ is chosen to be a large number, 
in either case $\mu_i' \notin (\mu_a', \mu_b')$ and constraint \eqref{eq:realize_constraints2} 
is satisfied. 
If constraint \eqref{eq:realize_constraints2} is satisfied, then a similar argument 
allows us to choose the value of $z_i$ that satisfies constraints 
\eqref{eq:realize_constraints3a} and \eqref{eq:realize_constraints3b}. 

\subsection{Validity of \cref{alg:gapucb}}
\label{sec:appendix gapucb algorithm proof}
In this section we find the value of $\gaprucb_a(t)$ as defined in \eqref{eq:realize} by first obtaining an upper bound to it. The proof of the upper bound is \emph{constructive} in nature, showing that the upper bound is actually achievable. That is, (a) there is a set of real numbers $\{\mu_i': i \in [K]\}$ which satisfy \eqref{eq:realize_constraints1}, (b) an index $a_\ast$ which satisfies \eqref{eq:realize_constraints2} with $x = \mu_{a_\ast}'$, such that $\gaprucb_a(t) = \mu_{a_\ast}' - \mu_a'$.

We first find an upper bound to the right gap of an arm $a$ assuming we know its true mean $\mu_a$, but only have confidence intervals for the means of the other arms $\mu_i \in [l_i(t), r_i(t)] \forall i \neq a$.
\begin{lemma}\label{propn:G^R}
If all the arm means are known, the right gap associated with an arm $a\in [K]$ is $\min_{i: \mu_i > \mu_a}\mu_i - \mu_a$; if the domain is empty we say that arm $a$'s right gap is $0$. For any $x \in \mathbb{R}$, define a function $G_a^r(\cdot)$ of the confidence intervals as follows.
\begin{equation*}
G_a^r(x, t) \triangleq
\begin{cases}
\min_{j: l_j(t) > x} r_j(t) - x &\text{if } \{j: l_j(t) > x\} \neq \emptyset,\\
\max_{j \neq a} r_j(t) - x &\text{otherwise}.
\end{cases}
\end{equation*}
Suppose we know the value of arm $a$'s mean, i.e.\ $\mu_a$ and the confidence intervals $[l_i(t), r_i(t)] \forall i \neq a$. Then the largest possible right gap of arm $a$ is $G_a^r(\mu_a, t)$.
\end{lemma}
\begin{proof}
We suppose that the right gap of arm $a$ is greater than the upper bound and show a contradiction to the good event \eqref{eq:valid confidence interval}.

\textbf{Case I:} $\{j: l_j(t) > \mu_a\} \neq \emptyset$. 
Identify the arm $j_\ast = \arg\min_{j: l_j(t) > \mu_a} r_j(t)$ such that $G_a^r(\mu_a,t) = r_{j_\ast}(t)-\mu_a$. Let the true right gap for arm $a$ be $\mu_k-\mu_a$. If $k = j_\ast$, then $\mu_k - \mu_a > r_{j_\ast}(t)-\mu_a$ would mean that $\mu_{j_\ast} > r_{j_\ast}(t)$, which is a contradiction. If $k \neq j_{\ast}$ and the right gap is $\mu_k-\mu_a$, then all arms $j \in [K]$ are such that $\mu_j \notin (\mu_a, \mu_k)$. But if $\mu_k-\mu_a > G_a^r(\mu_a, t)$ then $\mu_k > r_{j_\ast}(t)$, and from the domain in the definition of $j_\ast$, its left bound $l_{j_\ast}(t) > \mu_a$. Hence the confidence interval of $j_\ast$ satisfies $\mu_a < l_{j_\ast}(t) < r_{j_\ast}(t) < \mu_k$. If $\mu_{j_\ast} \notin (\mu_a,\mu_k)$ then $\mu_{j_\ast} \notin [l_{j_\ast}(t), r_{j_\ast}(t)]$ and that is a contradiction.

\textbf{Case II:} $\{j: l_j(t) > \mu_a\} = \emptyset$. 
Identify the arm $j_\ast = \arg\max_{j \neq a}r_j(t)$ such that $G_a^r(\mu_a, t) = r_{j_\ast}(t) - \mu_a$. Let the true right gap for arm $a$ be $\mu_k-\mu_a$. If $\mu_k - \mu_a > G_a^r(\mu_a, t)$ then $\mu_k > \max_{j \neq a}r_j(t)$ and that is a contradiction.

Thus the right gap of arm $a$ is at most $G_a^r(\mu_a, t)$. We can achieve this upper bound by choosing the set of means $\{\mu_i': i \in [K]\setminus a\}$ in the following manner. If the value of $G_a^r(\mu_a, t)$ is given by the first branch, set $\mu_i' = r_i(t) \forall i: r_i(t) > \mu_a$ and $\mu_i' = l_i(t) \forall i: l_i(t) < \mu_a$. Otherwise if the value is given by the second branch set $\mu_{a_\ast}' = r_{a_\ast}(t)$ for the arm $a_\ast \neq a$ which has the largest right bound, and set all other $\mu_i' = l_i(t)$ (\textit{c.f.\ }\cref{fig:rightgapucbfixedpos} in \cref{sec:prelim gap upper bound}).
\end{proof} 
The \emph{left gap} analog of the above proposition can also be proved in a similar manner as above.
\begin{lemma}\label{propn:G^L}
For any $x \in \mathbb{R}$ and arm $a \in [K]$, define a function $G_a^l(\cdot)$ of the confidence intervals as follows.
\begin{equation}\label{eq:G^L}
G_a^l(x, t) \triangleq
\begin{cases}
x - \max_{j: r_j(t) < x} l_j(t) &\text{if } \{j: r_j(t) < x\} \neq \emptyset,\\
x - \min_{j \neq a} l_j(t) &\text{otherwise}.
\end{cases}
\end{equation}
Suppose we know $\mu_a$. Using the confidence intervals $[l_i(t), r_i(t)] \forall i \neq a$, an upper bound to the left gap of arm $a$ is $G_a^l(\mu_a, t)$.
\end{lemma}
We now replace our knowledge of the true mean value $\mu_a$ by the good event fact that $\mu_a \in [l_a(t), r_a(t)]$ at all times $t$. The following lemma is instrumental in arriving at an upper bound for the right gap of arm $a$ that is consistent with the all the arms' confidence intervals.
\begin{lemma}\label{lem:left_bd_right_gap}
At time $t$, for any arm $a$ its true mean $\mu_a \in [l_a(t), r_a(t)]$ in the good event. Define a subset of arms $\mathcal{I}_a^R(t) \triangleq \{i: l_i(t) \in [l_a(t), r_a(t)]\}$ whose left bounds lie within the confidence interval of arm $a$. Consider a set of $K$ real numbers $\mathcal{P}' \triangleq \{\mu_i' \in [l_i(t), r_i(t)]:i \in [K]\}$, each associated with a corresponding arm. 
The largest value for the right gap of arm $a$ if the means are $\mathcal{P}'$, i.e.,
\begin{equation*}
\max \{\mu_i' - \mu_a' : \mu_i' > \mu_a', \nexists \mu_j' \in (\mu_a',\mu_i'), i, j \in [K]\setminus a\}
\end{equation*} 
occurs when $\mu_a' = l_i(t)$ for some $i \in \mathcal{I}_a^R(t)$.
\end{lemma}
\begin{proof}
Suppose the largest right gap occurs when $\mu_a' \neq l_i(t)$ for any $i \in \mathcal{I}_a^R(t)$. Note that $a \in \mathcal{I}_a^R(t)$ and hence the set is not empty. We show that the right gap can be larger while still satisfying event \eqref{eq:valid confidence interval}. Let $l_{i_a}(t) = \max_{i \in \mathcal{I}_a^R(t)}\{l_i(t) < \mu_a'\}$. Collect all arms in the set $\mathcal{J}_a = \{j : \mu_{j}' \in [l_{i_a}(t), \mu_a']\}$. Consider an alternate bandit model whose arm means are denoted by $\mathcal{Q}\triangleq \{q_i : i \in [K]\}$. We assign
\begin{align*}
q_i = l_{i_a}(t) \: \forall i \in \mathcal{J}_a \text{ and } q_i = \mu_i'\: \forall i \notin \mathcal{J}_a.
\end{align*}
This mean assignment satisfies $q_i \in [l_i(t), r_i(t)] \: \forall i \in [K]$. This is because by definition of arm $i_a$ in the original bandit model $\mathcal{P}'$, for all arms $j \in \mathcal{J}_a$ their left bounds satisfy $l_j(t) \leq l_{i_a}(t)$. Thus both the original $\mathcal{P}'$ and the alternate $\mathcal{Q}$ are possible bandit models in the good event \eqref{eq:valid confidence interval} up till current time $t$. 
However, the right gap for $a$ is larger in the alternate model $\mathcal{Q}$ as shown next. Let arm $i$ result in the right gap for $a$ in the original model $\mathcal{P}'$, i.e., the right gap is
\begin{equation*}
\mu_i' - \mu_a', \text{ and } \nexists \mu_j' \in (\mu_a', \mu_i').
\end{equation*}
Then in the alternate model, $q_i = \mu_i', q_a = l_{i_a}(t)$ and there is no mean $q_j \in (l_{i_a}(t), \mu_i')$. Then the right gap of arm $a$ is $\mu_i' - l_{i_a}(t) > \mu_i' - \mu_a'$. This contradicts the supposition that the right gap is the largest possible in the original bandit model $\mathcal{P}'$.
\end{proof}
An analogous lemma for the \emph{left gap} states that for any set of possible arm means $\mathcal{P}'$ that are consistent with the current confidence intervals, the largest possible left gap of arm $a$ occurs when $\mu_a' = r_i(t)$ for some arm $i \in \mathcal{I}_a^L(t) \triangleq \{i: r_i(t) \in [l_a(t), r_a(t)]\}$. Using the above, we can state the upper bound for the gap of an arm $a$ in terms of all the confidence intervals as follows.
\begin{theorem}\label{thm:gap_UCB}
At any time $t$, denote the upper bound to the right (\textit{resp.\ }left) gap of arm $a$ by $\gaprucb_a(t)$ (\textit{resp.\ }$\gaplucb_a(t)$). The expressions for these upper bounds in terms of the confidence intervals and the functions $G_a^r(\cdot), G_a^l(\cdot)$ in \cref{propn:G^R},~\cref{propn:G^L} are as follows.
\begin{align}
\gaprucb_a(t) &\triangleq \max\{G_a^r(l_j(t), t): l_j(t) \in [l_a(t), r_a(t)]\}, \nonumber\\
\gaplucb_a(t) &\triangleq \max\{G_a^l(r_j(t), t): r_j(t) \in [l_a(t), r_a(t)]\}. \label{eq:left_gap_UB}
\end{align}
Then an upper bound to the gap associated with arm $a$ at time $t$ is $\max\{\gaprucb_a(t), \gaplucb_a(t)\}$. \cref{alg:gapucb} gives pseudocode that evaluates $\gaprucb_a(t)$.
\end{theorem}
\begin{proof}
We argue for the right gap, an analogous proof gives the statement for the left gap. 
At any time $t$ in the good event $\mu_i \in [l_i(t), r_i(t)] \forall i {\in} [K]$, in particular any number in the range $[l_a(t), r_a(t)]$ can be potentially the mean of arm $a$. From \cref{lem:left_bd_right_gap}, we know that for a set of numbers $\mathcal{P}'$ that satisfy all current confidence intervals and also maximize the right gap for arm $a$, the value $\mu_a' = l_i(t)$ for some left bound $l_i(t) \in [l_a(t), r_a(t)]$. If $\mu_a' = l_i(t)$ then by \cref{propn:G^R} $G_a^r(l_i(t), t)$ is the largest possible value for arm $a$ in the bandit model $\mathcal{P}'$. Taking the maximum over all arms in the set $\mathcal{I}_a^R(t) = \{i \in [K]: l_i(t) \in [l_a(t), r_a(t)]\}$, we get the right gap upper bound $\gaprucb_a(t)$.

We note that the value $\gaprucb_a(t)$ is achievable by an assignment of means that satisfy the confidence bounds at time $t$. Without loss of generality, assume $\gapucb_a(t) = \gaprucb_a(t) = G_a^r(l_{a_\ast}(t), t)$ for some arm $a_\ast$. One can assign $\mu_a = l_{a_\ast}(t)$ and other means in a way similar to that in the proof of \cref{propn:G^R} to obtain a right gap for arm $a$ equal to the value $G_a^r(l_{a_\ast}(t), t)$. 
\end{proof}

\section{Details for \cref{sec:analysis}: Accuracy}
\TheoremAccuracy*
\begin{proof}
    Recall that the true maximum gap exists between arms $(m)$ and $(m+1)$. The algorithms return a wrong clustering $\gapucb_{(m)}(t) <
    \gaplcb(t)$ for any time $t$. We show that this leads to a contradiction if the good event \eqref{eq:valid confidence interval} holds. 

    Assume \eqref{eq:valid confidence interval} holds and $\gapucb_{(m)}(t) < \gaplcb(t)$ at some time $t$. 
        Recall that $\gaplcb(t)$ is computed using $\eqref{eq:max gap lower bound}$,
    and let $(s)_t$ be the maximizer in \eqref{eq:max gap lower
    bound}. Let $a$ be such that $a \in \{(1)_t,\dots,(s)_t\}$ and $a+1 \in
    \{(s+1)_t, \dots, (K)_t\}$. If $(3)$ holds, we have that
    \begin{align*}
        \Delta_{\max} &\le \gapucb_{(m)}(t) < \gaplcb(t) \overset{(a)}{\le} l_a(t) - r_{a+1}(t) \le \mu_a - \mu_{a+1},
    \end{align*}
    where (a) holds because $\gaplcb(t)$ is the minimum gap between a left
    confidence interval in $\{(1)_t, \dots, (s)_t\}$ and a right confidence
    interval in $\{(s+1)_t, \dots, (K)_t\}$. This contradicts the fact that
    $\Delta_{\max}$ is the largest gap.
\end{proof}

\section{Sample Complexity: Proof of \cref{thm:sample complexity}}
To state our sample complexity bounds we use a constant $\alpha$ defined as follows \citep{even2006action}.
\begin{remark}
    \label{remark:confidence interval sample complexity}
    There exists constant $\alpha$ such that for
    all $x>0$, if the number of samples $s \ge \alpha \frac{\log (K/\delta
    x)}{x^2}$, then $c_s \le x$, where $c_s$ is the confidence interval given by
    \eqref{eq:mean upper lower bound}.
\end{remark}
\subsection{Sample Complexity of $\maxgapelim$}
\label{sec:appendix sample complexity maxgapelim}
\textbf{Early Stopping Rule for Clustering}:
In the pseudocode in \cref{alg:maxgapelim}, $\maxgapelim$ stops when the size of
the active set $|A|\le 2$ (line $7$). However, if we are only interested in
clustering the arms according to the maximum gap and not interested in the
identities of the arms which share the maximum gap (arms $(m), (m+1)$), we can
stop earlier as follows. Assume that \eqref{eq:max gap lower bound} is
greater than 0 and  let $(k_\ast)_t$ be the maximizer.   This
partitions the arms into the sets $\{(1)_t,\dots, (k_\ast)_t\}$ and
$\{(k_\ast+1)_t,\dots,(K)_t\}$. 
$\maxgapelim$ can terminate when the maximum left gap of all
arms in $\{(1)_t,\dots, (k_\ast)_t\}$ and  
the maximum right gap of all arms in $\{(k_\ast+1)_t,\dots,K\}$ are both less
than the lower bound $\gaplcb(t)$. The termination condition can be
expressed as $S=1$, where
\begin{align}
    S = 1\{\gaprucb_a(t) < \gaplcb(t) , \, \forall\,a:l_a(t)\ge l_{(k_\ast)_t}(t \} \cdot
    1\{\gaplucb_a(t) < \gaplcb(t)  , \, \forall\,a:r_a(t) \le l_{(k_\ast+1)_t}(t) \}.
    \label{eq:stopping condition}
\end{align}

To account for the lower sample complexity as a result of the stopping rule for
clustering, we modify \eqref{eq:gammair definition} and \eqref{eq:gammail
definition} and define new parameters that yield an improved sample complexity
than that stated in \cref{thm:sample complexity}. Define
\begin{align}
    \elimparameter^r_a &= \max\big\{ &\max_{j:\Delta_{a,j}>0} \left(
    \min\{\Delta_{a,j}/4, ((\Delta_{\max}-\Delta_{a,j})/8)\} \right),
((\Delta_{\max}-\Delta_{a,1})/8) \big\} \label{eq:right gap parameter}, \\
    \elimparameter^l_a &= \max\big\{ &\max_{j:\Delta_{a,j}<0} \left(
                                      \min\{\Delta_{a,j}/4,
                                      ((\Delta_{\max}-\Delta_{j,a})/8)\}
                                      \right),
                                      ((\Delta_{\max}-\Delta_{a,K})/8)
                                  \big\},  \label{eq:left gap parameter}
\end{align}
where just like in \eqref{eq:gammair definition}, the maxima assumed to be infinity if there is no $j$ that
satisfies the constraint under the inner maximization. We define $\elimparameter_a =
\min\{\elimparameter_a^r, \elimparameter_a^l\}$ as before and state our improved sample
complexity bound for $\maxgapelim$ next.

\begin{theorem}
    With probability at least $1-\delta$, the sample complexity of $\maxgapelim$ is bounded by
    $$H = \alpha \sum_{\substack{a \in [K]: \\ 
    a \notin \{(m),(m+1)\}}} 
    \frac{\log (K/\delta \elimparameter_a)}{\elimparameter_a^2}.$$
    \label{thm:elimination sample complexity}
\end{theorem}
\begin{proof}
    Arm $a$ is eliminated in $\maxgapelim$ when $\gapucb_a(t) < \gaplcb(t)$, whee
    $\gapucb_a(t)$ is defined as the maximum of 
    the left and right gap upper bounds (see \cref{sec:prelim gap upper bound}). \cref{lem:right gap
      elimination} and \cref{lem:left gap elimination} 
    prove that the sufficient condition for each of these upper bounds to be les than
    $\gaplcb(t)$ is $c_{T_a(t)} \le \elimparameter_a$. The result then follows
    by \cref{remark:confidence interval sample complexity}.
\end{proof}
\begin{lemma}
    If the good event \eqref{eq:valid confidence interval} holds, then for all $a \in [K]$, for all $t \in \mathbb{N}$,
    \begin{align*}
        l_a(t) \ge \mu_a-2c_{T_a(t)} \text{ and } r_a(t) \le \mu_a+2c_{T_a(t)}
    \end{align*}
    where $c_s = \sqrt{\frac{\beta_\delta(s)}{s}}$.
    \label{lem:simple ci bounds}
\end{lemma}

\begin{proof}
    We have 
    \begin{align*}
        \hat{\mu}_a(t) + c_{T_a(t)} \overset{(a)}{\ge} \mu_a \Rightarrow l_a(t) = \hat{\mu}_a(t)-c_{T_a(t)} &\ge \mu_a-2c_{T_a(t)}. 
    \end{align*}
    Similarly,
    \begin{align*}
        \hat{\mu}_a(t)-c_{T_a(t)} \overset{(a)}{\le} \mu_a \Rightarrow r_a(t) = \hat{\mu}_a(t)+c_{T_a(t)} &\le \mu_a+2c_{T_a(t)}.
    \end{align*}
    In both the equations above, $(a)$ holds by \eqref{eq:valid confidence interval}.
\end{proof}

\begin{lemma}
    \label{lem:right gap elimination}
    Assume \eqref{eq:valid confidence interval} holds, and consider $a \neq
    m+1$. In \maxgapelim if $t$ is such that $c_{T_a(t)} \le \elimparameter^r_a$, then
    $$\gaprucb_a(t) < \gaplcb(t).$$ 
\end{lemma}

\begin{proof}
    Note that at time $t$ in \cref{alg:maxgapelim}, $T_a(t)=t$ and
    $c_{T_a(t)} = c_t$ for all arms $a
    \in A$. Assume \eqref{eq:valid confidence interval} holds. We have
    $c_t < \elimparameter^r_a < \Delta_{\max}/4$. This implies that 
    \begin{align}
        \begin{aligned}
            l_m(t) &\overset{(a)}{\ge} \mu_m-2c_t = \mu_{m+1}+\Delta_{\max}-2c_t  \overset{(a)}{\ge} r_{m+1}(t) + \Delta_{\max} - 4c_t \ge r_{m+1}(t).
        \end{aligned}
        \label{eq:lem right gap elimination eqn 1}
    \end{align}
    where (a) holds by \cref{lem:simple ci bounds}.

    From \eqref{eq:lem right gap elimination eqn 1} we have that 
    \begin{align}
        \gaplcb(t) \ge l_m(t) - r_{m+1}(t) \ge \Delta_{\max}-4c_t \label{eq:lower bound on lower bound}
    \end{align}

    Recall from \eqref{eq:right gap parameter} that for $a \ne 1$, 
\begin{align}
    \elimparameter^r_a = \max\big\{ &\max_{j:\Delta_{a,j}>0} \left(
    \min\{\Delta_{a,j}/4, ((\Delta_{\max}-\Delta_{a,j})/8)\} \right),
((\Delta_{\max}-\Delta_{a,1})/8) \big\}.
        \label{eq:lem right gap elimination eqn 2}
\end{align}
    %

There are two terms in $\elimparameter_a^r$ and $c_t$ could be less than either
of these terms. First, suppose that 
$$c_t < \max_{j:\Delta_{a,j}>0} \left(
\min\{\Delta_{a,j}/4, ((\Delta_{\max}-\Delta_{a,j})/8)\} \right),$$ 
and let 
    \begin{equation}
        e = \argmax_{j:\Delta_{a,j}>0} \left( \min\{\Delta_{a,j}/4, ((\Delta_{\max}-\Delta_{a,j})/8)\} \right).
        \label{eq:lem right gap elimination eqn 3}
    \end{equation}
    For any arm $j$ such that $\Delta_{\max} < \Delta_{a,j}$, the inner minimum in \eqref{eq:lem right gap elimination eqn 3} will be negative. On the other hand, since $a \ne m+1$, there must exist an arm $j$ such that $\Delta_{\max} > \Delta_{a,j}$, and for such an arm $j$ the inner minimum will be positive. Since $e$ is the arm that maximizes the inner minimum, the inner minimum must be positive for $e$. Thus we have that $\Delta_{\max} > \Delta_{a,e}$. 
    
    From \eqref{eq:lem right gap elimination eqn 2}, \eqref{eq:lem right gap
    elimination eqn 3}, we have that 
    \begin{equation}
        c_t < \Delta_{a,e}/4 \quad \text{and} \quad c_t < (\Delta_{\max}-\Delta_{a,e})/8.
        \label{eq:lem right gap elimination eqn 4}
    \end{equation}

    Since $c_t < \Delta_{a,e}/4$, by following an argument similar to
    \eqref{eq:lem right gap elimination eqn 1} we have that $l_e(t) \ge r_a(t)$,
    and hence the first branch of \eqref{eq:G^R} will be used to compute
    $\gaprucb_a(t)$. Hence we have 
    \begin{align*}
        \gaprucb_a(t) &\overset{(a)}{\le} r_e(t)-l_a(t) \overset{(b)}{\le}
        \Delta_{a,e} + 4c_t \overset{(c)}{\le} \Delta_{\max} - 4c_t
        \overset{(d)}{\le} \gaplcb(t)
    \end{align*}
    where $(a)$ follows from \eqref{eq:G^R} and \eqref{eq:right_gap_UB}, $(b)$ holds from \cref{lem:simple ci bounds}, $(c)$ follows by \eqref{eq:lem right gap elimination eqn 4}, and $(d)$ holds by \eqref{eq:lower bound on lower bound}.

    For the second case, assume 
    $$c_t < (\Delta_{\max}-\Delta_{a,1})/8.$$
    Let $e = \argmax_{i \ne a} r_i(t)$. From \eqref{eq:right_gap_UB}, we have that
    $$\gaprucb_a(t) \le r_e(t)-l_a(t) \overset{(a)}{\le} \Delta_{a,e}+4c_t \le
    \Delta_{a,1}+4c_t \overset{(b)}{\le} \Delta_{\max} - 4c_t \overset{(c)}{\le}
    L(t),$$
    where (a) holds by \cref{lem:simple ci bounds}, (b) holds by the case
    assumption, and (c) holds by \eqref{eq:lower bound on lower bound}.
\end{proof}

\begin{lemma}
    \label{lem:left gap elimination}
    Assume \eqref{eq:valid confidence interval} holds, and consider $a \ne m$.
    In \maxgapelim if $t$ is such that $c_t \le \elimparameter_a^l$, then
    $$\gaplucb_a(t)< \gaplcb(t)$$ 
\end{lemma}
\begin{proof}
    The proof is analogous to the proof of \cref{lem:right gap elimination}.
\end{proof}

\subsection{Sample Complexity of $\maxgapucb$}
\label{sec:maxgapucb lemmas}

For the sample complexity analysis of \maxgapucb, we use a modified version of the left and right confidence bounds introduced in \eqref{eq:mean upper lower bound}. We redefine
\begin{equation}\label{eq:modified ci}
l_i'(t) \triangleq \max_{s \leq t} l_i(s), \qquad r_i'(t) \triangleq \min_{s \leq t} r_i(s).
\end{equation}
The nice property that these bounds have is that $[l_i'(t), r_i'(t)] \subseteq [l_i(s), r_i(s)]$ for all $t \geq s$. \cref{lem:nested_CIs} shows that these modified bounds retain the same confidence guarantee for the arm mean values as the original confidence bounds. In what follows, we will exclusively use the modified confidence bounds (except in \cref{lem:nested_CIs} where we show they are correct). We drop the prime symbol in their notation for brevity and henceforth $l_i(t), r_i(t)$ denote the modified confidence bounds given in \eqref{eq:modified ci}. 

 We state and prove our main sample complexity result in \cref{thm:ucb sample complexity}.
\begin{theorem}
    \label{thm:ucb sample complexity}
    With probability at least $1-\delta$, the number of times $\maxgapucb$
    samples a sub-optimal arm, i.e.\ an arm $i \not\in \{(m), (m{+}1)\} $, is upper bounded by $6\alpha\ucbparameter_i^{-2}\log(K/\delta \ucbparameter_i)$. The constant $\alpha$ is defined in \cref{remark:confidence interval sample complexity}. Thus, the number of times $\maxgapucb$ samples suboptimal arms is  
    $$H = 6\alpha\sum_{\substack{i \in [K]: \\ 
    i \notin \{(m),(m+1)\}}} \frac{\log(K/\delta \ucbparameter_i)}{\ucbparameter_i^2}.$$
\end{theorem}
\begin{proof}
    We show that the result holds true as long as the confidence
    intervals for the means are correct \eqref{eq:valid confidence interval}. Let 
\begin{equation}
    \tau_r = \alpha \frac{\log (K/\delta \ucbparameter_i^r)}{(\ucbparameter_i^r)^2}, \quad \text{and} \quad \tau_l = \alpha \frac{\log (K/\delta \ucbparameter_i^l)}{(\ucbparameter_i^l)^2},
    \label{eq:taur taul definition}
\end{equation}
where $\alpha$ is defined in \cref{remark:confidence interval sample
  complexity}.  Note that
  $\gapucb_{(m)}(t) = \gapucb_{(m+1)}(t) \ge \Delta_{\max}\,\forall\,t$. Arm
$i$ is sampled either because $\gaprucb_i$ is the largest or
$\gaplucb_i$ is the largest. We prove in \cref{lem:ucbrightgap less
  than deltamax} below that when $i$ is sampled $3\tau_r$ times due to its right gap, $\gaprucb_i < \Delta_{\max}$. Hence $\maxgapucb$ will
not sample $i$ due to its right gap more than $3\tau_r$ times because
beyond this point $\gapucb_{(m)}$ will be higher. It can similarly be
proved that when $i$ is sampled $3\tau_l$ times due to its left gap,
$\gaplucb_i < \Delta_{\max}$. Thus, arm $i$ will be sampled at most
$3(\tau_r+\tau_l) \leq 6\max\{\tau_r, \tau_l\}$ times.
\end{proof}
To ease the explanation, we only focus on the right gap of $i$ from
here onwards and set 
\begin{equation}
    \tau = \alpha \frac{\log (K/\delta \ucbparameter_i^r)}{(\ucbparameter_i^r)^2}.
    \label{eq:tau definition}
\end{equation}
    Furthermore, in the lemmas below, we only focus on samples of $i$
    drawn when $\gaprucb_i$ was the largest upper bound. With a slight overload of notation, let $t(i,s)$ denote the (random) smallest time when arm $i$ has been sampled $s$ times by $\maxgapucb$ (owing to its right gap). 

\begin{lemma}
With probability $1-\delta$, $\gaprucb_i(t(i,3\tau)) < \Delta_{\max}.$
\label{lem:ucbrightgap less than deltamax}
\end{lemma}
\begin{proof} Since the proof is long and technical, we first give an outline of the entire proof. 

\begin{figure}[t]
    \centering
    \includegraphics[width=.5\linewidth]{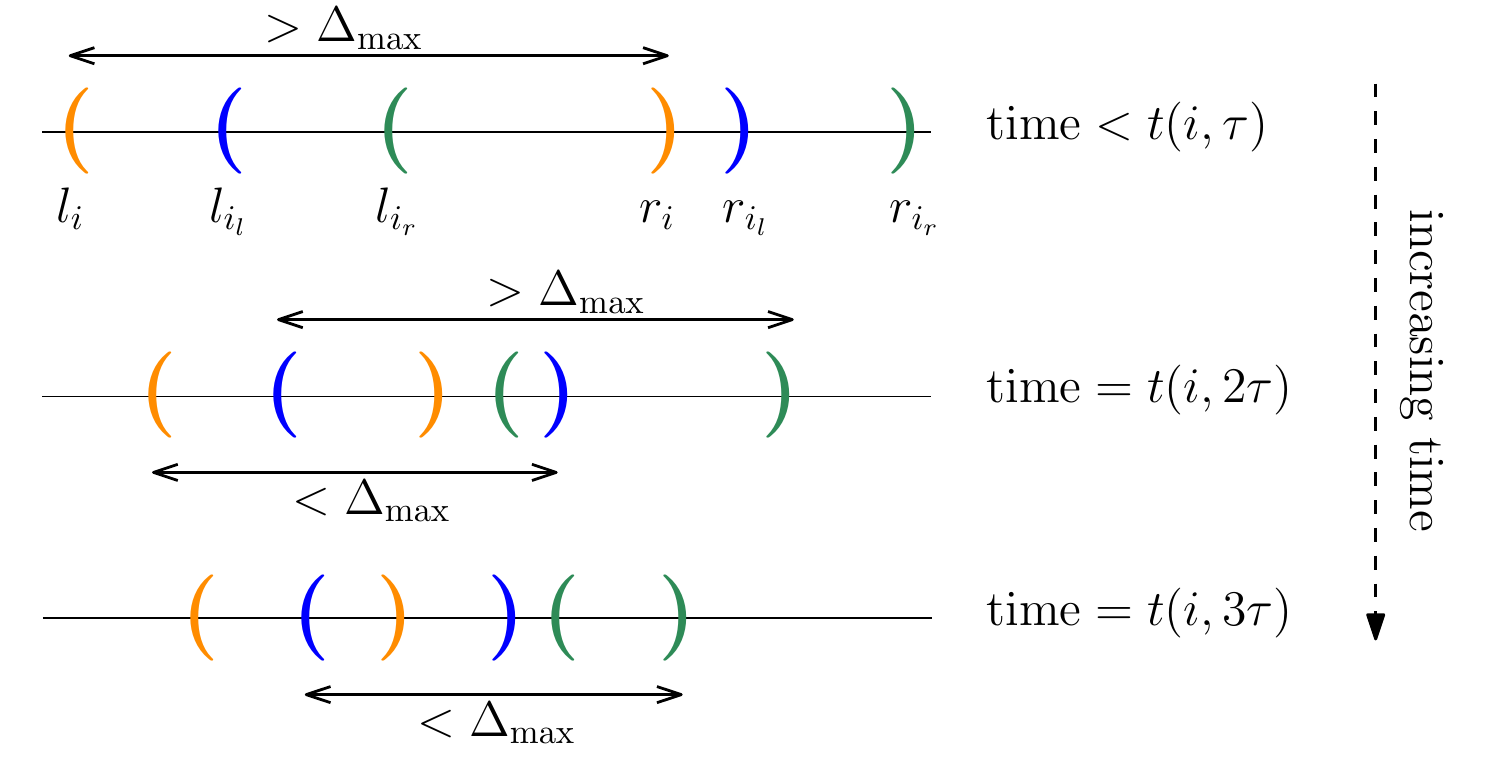}
\caption{Illustration of left and right confidence bounds during a run of $\maxgapucb$ at three different times, the argument $t$ for the bounds are omitted. Arms $i_l, i_r$ are such that $\gaprucb_i(t(i, 3\tau)) = r_{i_r}(t(i, 3\tau)) - l_{i_l}(t(i, 3\tau))$.}
\label{fig:UCB_sample_complexity}
\end{figure}
    {\em Outline:} Define $i_r^t$ and $i_l^t$ to be the arms that form the left and right boundaries of $\gaprucb_i(t)$ (the two arms that result in the maximum value of \eqref{eq:right_gap_UB} at $t$). 
By this definition, 
$
\gaprucb_i(t) = G^r_i(l_{i_l^t}(t), t) = r_{i_r^t}(t) - l_{i_l^t}(t)
$. 
Consider the arms used in computing $\gaprucb_i(t(i, 3\tau))$, i.e.\ 
$i_l^{t(i,3\tau)}, i_r^{t(i, 3\tau)}$, and denote them as $i_l, i_r$ for brevity.  
\cref{fig:UCB_sample_complexity} shows the 
confidence intervals of $i_l$ in blue and those of $i_r$ in green. 
Initially the confidence intervals 
are large, i.e., the width between the right and left bounds of arm $i$ is greater than 
$\Delta_{\max}$ before time $t(i, \tau)$. After $t(i, 2\tau)$ rounds of
$\maxgapucb$, the confidence interval of $i$ will have shrunk. However, note that 
the right gap of arm $i$ involves either $i_l$ and/or $i_r$. 
Since $\maxgapucb$ samples \emph{all} arms that can attain the highest gap upper bound, 
it turns out that it will also  
sample $i_l$ enough times to make $r_{i_l}(t(i, 2\tau)) - l_i(t(i, 2\tau)) < \Delta_{\max}$. 
If $i$ is still sampled after $t(i, 2\tau)$ rounds due to its right gap, then its gap upper bound 
must involve an arm which is disjoint from $i$'s confidence interval, such as the arm $i_r$. Then from 
$t(i, 2\tau)$ to $t(i, 3\tau)$, $\maxgapucb$ samples $i_l$ and $i_r$ enough times to make 
$\gaprucb_i(t(i, 3\tau)) = r_{i_r}(t(i, 3\tau)) - l_{i_l}(t(i, 3\tau))
< \Delta_{\max}$.  
%

We divide the proof into four parts. In the first part, we divide all the arms into subsets (which we refer to as levels). These subsets are defined such that arms within a subset obey collective properties, that we study in some of the subsequent lemmas\todos{lemmas section}. In the second and third part, we prove that arms $i_r^{t(i,3\tau)}$ and $i_l^{t(i,3\tau)}$ are always sampled whenever $i$ is sampled from $[t(i,2\tau), t(i,3\tau)]$. Finally in part four, we use these arms to argue that $\gaprucb_i(t(i,3\tau)) < \Delta_{\max}$.\\[-5pt]

\noindent {\bf Level Sets}:\\
At any time $t$, we can identify three subsets of arms with respect to arm $i$ that we refer to as level $0$, level $1$, and level $2$ arms respectively, and argue that the arms that define $\gaprucb_i(t)$ must lie in one of these subsets. These levels sets are defined as follows. Let
\begin{align}
\mathcal{A}_i^0(t) &= \{a\in [K]: l_i(t) \leq r_a(t) < r_i(t)\}, \label{eq:level0}\\
\mathcal{A}_i^1(t) &=\{a \in [K]: l_a(t) \leq r_i(t) \leq r_a(t)\}, \label{eq:level1}\\
\mathcal{A}_i^2(t) &= \left\lbrace a \in [K]: r_i(t) < l_a(t) \leq \min_{j: l_{j}(t) > r_i(t)}r_{j}(t)\right\rbrace .\label{eq:level2}
\end{align}
From their definitions the three subsets are pairwise disjoint at every $t \in \mathbb{N}$. Let 
\begin{equation}
    \mathcal{A}_i(t) = \mathcal{A}_i^0(t) \cup \mathcal{A}_i^1(t) \cup \mathcal{A}_i^2(t) 
    \label{eq:union of levels definition}
\end{equation}
denote the union of the three levels. 
From the definition of $\gaprucb_i(t)$ in \eqref{eq:right_gap_UB} and \eqref{eq:level0}, \eqref{eq:level1}, the arm 
\begin{equation}
    i_l^t \in \mathcal{A}_i^0(t) \cup \mathcal{A}_i^1(t) \,\forall\,t. \label{eq:iL is level0 or level1}
\end{equation}
\cref{lem:level3+_dont_matter} proves that the arm $i_r^t \in \mathcal{A}_i(t) \,\forall\,t$. Thus at any time $t$, only arms in $\mathcal{A}_i(t)$ are relevant for the right gap of arm $i$.





\todos{move proof of $\mathcal{A}_i^0$ also to Lemma 11?}

Suppose $t(i,3\tau) < \infty$, i.e., arm $i$ is sampled at least $3\tau$ times. To avoid clutter,we let 
$$i_r = i_r^{t(i,3\tau)} \quad \text{and} \quad i_l = i_l^{t(i,3\tau)},$$ 
and use the full notation $i_r^t$ for $t \ne t(i,3\tau)$. We next argue that $i_r$ and $i_l$ must be sampled $\tau$ times before $t(i,3\tau)$.

\noindent {\boldmath\bfseries $i_r$ must have been sampled at least $\tau$ times before $t(i,3\tau)$}: \\
By \cref{lem:level3+_dont_matter}, $i_r \in \mathcal{A}_i(t(i,3\tau))$. From \cref{crlry:level3+_dont_become_2}, 
$i_r \in \mathcal{A}_i(s) \,\forall\,s \in [t(i,2\tau), t(i,3\tau)]$. If $i_r \in \mathcal{A}_i^0(s) \cup \mathcal{A}_i^1(s)$ for any $s\in [t(i,2\tau), t(i,3\tau)]$, then $r_{i_r}(t(i,3\tau)) - l_i(t(i,3\tau)) \le r_{i_r}(s) - l_{i}(s) < \Delta_{\max}$ by \cref{lem:nested_CIs} and \cref{lem:i_R_cant_be_level1_at_>tau'}, and we are done. Let us hence look at the case when $i_r \in \mathcal{A}_i^2(s) \,\forall\,s \in [t(i,2\tau), t(i,3\tau)]$. We have by \cref{lem:i_R_cant_be_level1_at_>tau'} that $i_r^s \in \mathcal{A}_i^2(s) \,\forall\,s\in [t(i,2\tau), t(i,3\tau)]$, and \cref{lem:all_in_level_sampled} then implies that $i_r$ must be sampled whenever $i$ was sampled for $s\in [t(i,2\tau), t(i,3\tau)]$. Hence $i_r$ is sampled at least $\tau$ times before $t(i,3\tau)$.\\[-10pt]

\noindent {\boldmath\bfseries $i_l$ must have been sampled at least $\tau$ times before $t(i, 3\tau)$}:\\
From \cref{crlry:level3+_dont_become_2}, since the level of an arm cannot decrease from 2 to 1, \todos{lemma that says level cannot decrease from 2 missing? A; Cor~2 gives it?} $i_l \in \mathcal{A}_i^0(s) \cup \mathcal{A}_i^1(s)$ for all $s \in [t(i,2\tau), t(i,3\tau)]$. If $i_l \in \mathcal{A}_i^1(s)\,\forall\,s\in [t(i,2\tau), t(i,3\tau)]$, then $i_l$ is sampled $\tau$ times whenever $i$ is sampled by \cref{lem:all_in_level_sampled}. \\
On the other hand, if $i_l \in \mathcal{A}_i^0(s)$ for some $s \in [t(i,2\tau), t(i,3\tau)]$, let $\gaprucb_i(s) = r_{i_r^s}(s) - l_{i_l^s}(s)$. We consider two cases, $r_{i_l}(s) < l_{i_l^s}(s)$ and $r_{i_l}(s) \ge l_{i_l^s}(s)$. First, if $r_{i_l}(s) < l_{i_l^s}(s)$, then $l_{i_l}(s') < r_{i_l}(s') < l_{i_l^s}(s')$ for all $s' \geq s$ by \cref{lem:nested_CIs}. Since $i_l^s \in \mathcal{A}_i^0(s) \cup \mathcal{A}_i^1(s)$, we have by \cref{lem:nested_CIs} and \cref{lem:i_R_cant_be_level1_at_>tau'} that 
\begin{equation}
    r_{i_l^s}(t(i,3\tau)) - l_{i_l}(t(i,3\tau)) \leq r_{i_l^s}(t(i,3\tau)) - l_{i}(t(i,3\tau)) \le \Delta_{\max}.
    \label{eq:i_L_must_be_sampled_eq1}
\end{equation}
By the definition of $\gaprucb_i$ in \eqref{eq:right_gap_UB} we have that 
\begin{equation}
    \gaprucb_i(t(i,3\tau)) = G_i^r(l_{i_l}(t(i,3\tau)), t(i,3\tau)) \le r_{i_l^s}(t(i,3\tau)) - l_{i_l}(t(i,3\tau)).
    \label{eq:i_L_must_be_sampled_eq2}
\end{equation}
\eqref{eq:i_L_must_be_sampled_eq2} and \eqref{eq:i_L_must_be_sampled_eq1} imply that $\gaprucb_i(t(i,3\tau)) \le \Delta_{\max}$, and we are done. 
For the second case, suppose $r_{i_l}(s) \ge l_{i_l^s}(s)$. 
Lemma~\ref{lem:arms_that_can_access_sampled} then gives that arm $i_l$ is also sampled at time $s$. 
Thus, we have shown that either $\gaprucb_i(t(i,3\tau)) < \Delta_{\max}$, or $i_l$ is sampled whenever $i$ is sampled in $[t(i,2\tau), t(i,3\tau)]$.

We now show that $\gaprucb_i(t(i,3\tau)) < \Delta_{\max}$.\\[-10pt]

\noindent {\boldmath\bfseries $\gaprucb_i(t(i,3\tau)) < \Delta_{\max}$}:\\
Recall that $T_i(t(i,3\tau)), T_{i_l}(t(i,3\tau)), T_{i_r}(t(i,3\tau))$ are all larger than $\tau$. 
Let 
$$j_\ast = \argmax_{j: 0 < \Delta_{i,j} < \Delta_{\max}} \min \{\Delta_{i,j}/4, (\Delta_{\max} - \Delta_{i,j})/4\}$$
be the maximizer in \eqref{eq:gammair definition}, and note that $\mu_i < \mu_{j_\ast}$ by definition. 
Also note that $\tau$ and $\ucbparameter_i^r$ are defined in \eqref{eq:tau definition} and \eqref{eq:gammair definition} respectively so that 
\begin{equation}
    4c_{\tau} \le \Delta_{\max}-\Delta_{i,j_\ast} \quad \text{ and } \quad 4c_{\tau} \le \Delta_{i,j_\ast}
    \label{eq:tau defined so that}
\end{equation}

We split the proof into various cases depending on the ordering of the means $\mu_i, \mu_{i_l}, \mu_{i_r}, \mu_{j_\ast}$. First, note that if $\mu_{i_r} \le \mu_{i_l}$, then 
\begin{align*}
    \gaprucb_i(t(i,3\tau)) = r_{i_r}(t(i,3\tau)) - l_{i_l}(t(i,3\tau)) \le \mu_{i_r}-\mu_{i_l} + 4c_{\tau} \le \Delta_{\max}
\end{align*}
by \eqref{eq:tau defined so that}.
Second, if $\max\{\mu_i, \mu_{i_l}\} < \mu_{i_r} < \mu_{j_\ast}$, then
\begin{align*}
    \gaprucb_i(t(i,3\tau)) &\le r_{i_r}(t(i,3\tau)) - l_{i_l}(t(i,3\tau)) \le r_{i_r}(t(i,3\tau)) - l_i(t(i,3\tau)) \\ 
    &\le \mu_{i_r}-\mu_i+4c_{\tau} \le \Delta_{i,j_\ast} + 4c_{\tau} \le \Delta_{\max}
\end{align*}
by \eqref{eq:tau defined so that}. Third, we show that it cannot be the case that $\mu_i < \mu_{j_\ast} < \mu_{i_l} < \mu_{i_r}$. Assume to the contrary. This implies that
$$l_{i_l}(t(i,3\tau)) - r_i(t(i,3\tau)) \ge \mu_{i_l}-\mu_i-4c_{\tau} \ge \mu_{j_\ast}-\mu_i-4c_{\tau} > 0,$$
which contradicts \cref{eq:iL is level0 or level1}. Fourth, it cannot be the case that $\mu_{i_l} < \mu_{i_r} < \mu_i < \mu_{j_\ast}$, because $i_r \in \mathcal{A}_i^2(t(i,3\tau))$ by \cref{lem:i_R_cant_be_level1_at_>tau'}. The only case that remains is $\max\{\mu_i, \mu_{i_l}\} < \mu_{j_\ast} < \mu_{i_r}$, which we prove next by showing that $T_{j_\ast}(t(i,3\tau)) \ge \tau$.

\todos{confirm that we haven't assume $i_L < i_R$ anywhere in the proof}

%
$\underline{\max\{\mu_i, \mu_{i_l}\} < \mu_{j_\ast^r} < \mu_{i_r}}$: 

For any time $s \in [t(i, 2\tau), t(i, 3\tau)] $ such that $j_\ast \in \mathcal{A}_i^1(s) \cup \mathcal{A}_i^2(s)$, we have by Lemma~\ref{lem:i_R_cant_be_level1_at_>tau'} and \cref{lem:all_in_level_sampled} that $j_\ast$ is sampled whenever $i$ is sampled. Thus we only need to focus on times $s$ when $j_\ast \in \mathcal{A}_i^0(s)$.

Suppose now that $j_\ast \in \mathcal{A}_i^0(s)$ for some $s \in [t(i, 2\tau), t(i, 3\tau)]$ when $i$ was sampled and $\gaprucb_i(s) = r_{i_r^s}(s) - l_{i_l^s}(s)$. Recall that $i_l^s \in \mathcal{A}_i^0(s) \cup \mathcal{A}_i^1(s)$. We consider two cases depending on whether $l_{i_l^s}(t(i, 3\tau)) > l_{i_l}(t(i, 3\tau))$ or $l_{i_l^s}(t(i, 3\tau)) \le l_{i_l}(t(i, 3\tau))$.

\begin{itemize}
 \item $l_{i_l^s}(t(i, 3\tau)) > l_{i_l}(t(i, 3\tau))$: We have
     \begin{align*}
         \gaprucb_i(t(i,3\tau)) &= G(l_{i_l}(t(i,3\tau)), t(i,3\tau)) \stackrel{(a)}{\le} r_{i_l^s}(t(i, 3\tau)) - l_{i_l}(t(i, 3\tau)) \\
         &\le r_{i_l^s}(t(i, 3\tau)) - l_{i}(t(i, 3\tau)) \stackrel{(b)}{\le} r_{i_l^s}(s) - l_{i}(s) \stackrel{(c)}{\le} \Delta_{\max},
     \end{align*}
     where $(a)$ holds by \eqref{eq:G^R}, $(b)$ holds by \cref{lem:nested_CIs}\todos{the nested lemma should be written as $l(t) < l(t'), r(t) > r(t')$}, and $(c)$ holds by \cref{lem:i_R_cant_be_level1_at_>tau'}.

 \item $l_{i_l^s}(t(i, 3\tau)) \le l_{i_l}(t(i, 3\tau))$: Since $\max\{\mu_i, \mu_{i_l}\} < \mu_{j_\ast}$, we have $l_{i_l}(t) \le r_{j_\ast}(t) \,\forall\,t$. Hence, $l_{i_l^s}(t(i,3\tau)) < r_{j_\ast}(t(i,3\tau))$, and \cref{lem:nested_CIs} implies that $l_{i_l^s}(s) \le r_{j_\ast}(s)$. Recall that $s$ is a time such that $j_\ast \in \mathcal{A}_i^0(s)$, and hence 
     $$r_{j_\ast}(s)-l_{i_l^s}(s) \le r_{j_\ast}(s)-l_i(s) \le \Delta_{\max}.$$ 
     Now, since $i$ is sampled at time $s$, we have $\gaprucb_i(s) > \Delta_{\max}$, and \eqref{eq:G^R} then implies that $l_{j_\ast}(s) < l_{i_l^s}(s)$. 
     Hence by Lemma~\ref{lem:arms_that_can_access_sampled} $\maxgapucb$ must also sample arm $j_\ast$ at time $s$.
     \todos{G notation needs to be changed}
\end{itemize}
This proves that $T_{j_\ast^r}(t(i, 3\tau)) \geq \tau$. We use this to prove that $\gaprucb_i(t(i,3\tau)) < \Delta_{\max}$ as follows. First note that 
\todos{add discussion somewhere about why $j_\ast$ has been chosen this way}
\begin{align*}
l_{j_\ast^r}(t(i, 3\tau)) - r_i(t(i, 3\tau)) \geq p_{j_\ast^r} - p_i - 4c_{\tau} \ge 0.
\end{align*}
Second, since arm $i_l \in \mathcal{A}_i^0(t(i, 3\tau)) \cup \mathcal{A}_i^1(t(i, 3\tau))$, and hence 
$$l_{i_l}(t(i,3\tau)) < r_i(t(i,3\tau)) \le l_{j_\ast}(t(i,3\tau)).$$ 
Hence
\begin{align*}
\gaprucb_i(t(i, 3\tau)) &= G_i^r(l_{i_l}(t(i, 3\tau)), t(i, 3\tau)) \leq r_{j_\ast^r}(t(i, 3\tau)) - l_{i_l}(t(i, 3\tau)) \\ 
&\leq r_{j_\ast^r}(t(i, 3\tau)) - l_i(t(i, 3\tau)) \leq \mu_{j_\ast^r} - \mu_i + 4c_{\tau} \le \Delta_{\max}.
\end{align*}
\end{proof}

\begin{lemma}\label{lem:nested_CIs}
Over the sigma-algebra generated by all the arm rewards up till any time $t \in \mathbb{N}$, we have that
\begin{equation}\label{eq:nested_CIs}
\mathbb{P}\left( \forall t \in \mathbb{N}, \forall i \in [K], \mu_i \in \left[\max_{t'\leq t}l_i(t'), \min_{t' \leq t}r_i(t')\right]\right) 
= \mathbb{P}(\forall t \in \mathbb{N}, \forall i \in [K], \mu_i \in [l_i(t), r_i(t)]).
\end{equation}
\end{lemma}
\begin{proof}
Let $E'$ be the event in the LHS of \eqref{eq:nested_CIs} and let $E$ be the good event. First we show that $E' \subseteq E$. The event $E'$ implies that at any time $t$ and for any arm $i$, we have that
\begin{equation*}
\mu_i \in \left[\max_{t'\leq t}l_i(t'), \min_{t' \leq t}r_i(t')\right]
\implies \mu_i \in [l_i(t'), r_i(t')] \:\forall t' \leq t.
\end{equation*}
Hence the good event is true in this case.

Now we show that $E \subseteq E'$. Suppose that $E'$ is not true, so there is a time $t$ and arm $i$ such that $\mu_i \notin \left[\max_{t'\leq t}l_i(t'), \min_{t' \leq t}r_i(t')\right]$. Choose two time instants $s_l, s_r \in \mathbb{N}$ such that $s_l \in \arg\max_{t'<t}l_i(t'), s_r \in \arg\min_{t'<t}r_i(t')$. Then the supposition implies that either
\begin{equation*}
\mu_i \notin [l_i(s_l), r_i(s_l)] \quad \text{ or/and } \quad \mu_i \notin [l_i(s_r), r_i(s_r)].
\end{equation*}
Either of the above statements imply that the good event is not true. Hence $E \implies E'$. 
\end{proof}
\begin{corollary}\label{crlry:UCB_decreases}
For any two time instants $s, t \in \mathbb{N}$ if $s < t$ then $\gaprucb_i(s) \geq \gaprucb_i(t)$.
\end{corollary}
\begin{proof}
The quantity $\gaprucb_i(t)$ is defined in \eqref{eq:realize} as an optimization problem over a set of $K$ real numbers $\mathcal{P}' = \{\mu_i' \in [l_i(t), r_i(t)] : i \in [K]\}$. For a time $s < t$, the $\gaprucb_i(s)$ is an optimization over $\mathcal{P}'' = \{\mu_i'' \in [l_i(s), r_i(s)]: i \in [K]\}$. Lemma~\ref{lem:nested_CIs} states that $[l_i(t), r_i(t)] \subseteq [l_i(s), r_i(s)]$, hence we have that $\gaprucb_i(s) \geq \gaprucb_i(t)$.
%
\end{proof}
\begin{corollary}\label{crlry:level3+_dont_become_2}
For all $k \in [K]$ if $k \in \mathcal{A}_i^2(t)$ then $k \in \mathcal{A}_i(s)$ at all time instants $s \leq t$. If $k \in \mathcal{A}_i^2(t)$ then $k \notin \mathcal{A}_i^0(s') \cup \mathcal{A}_i^1(s')$ at all $s' \geq t$.
\end{corollary}
\begin{proof}
Define $\mathcal{J}(t) \triangleq \{j \in [K]: l_j(t) > r_i(t)\}$. For any $s \leq t$, if $j \in \mathcal{J}(s)$ then using Lemma~\ref{lem:nested_CIs},
\begin{equation}\label{eq:level2_domain}
l_j(t) \geq l_j(s) > r_i(s) \geq r_i(t) \implies j \in \mathcal{J}(t).
\end{equation} 
Hence if $k \in \mathcal{A}_i^2(t)$, from \eqref{eq:level2} we have that $l_k(t) \leq \min_{j \in \mathcal{J}(t)}r_j(t)$ and we get
\begin{equation*}
l_k(s) \leq l_k(t) \leq \min_{j \in \mathcal{J}(t)}r_j(t) \overset{\text{(a)}}{\leq} \min_{j \in \mathcal{J}(t)}r_j(s) \overset{\text{(b)}}{\leq} \min_{j \in \mathcal{J}(s)}r_j(s),
\end{equation*}
where the inequality (a) is true because of Lemma~\ref{lem:nested_CIs} and inequality (b) is true as $\mathcal{J}(s) \subseteq \mathcal{J}(t)$ by \eqref{eq:level2_domain}. This implies that $k \in \mathcal{A}_i^0(s) \cup \mathcal{A}_i^1(s) \cup \mathcal{A}_i^2(s) = \mathcal{A}_i(s)$. 

If $k \in \mathcal{A}_i^2(t)$ we have that $r_i(t) < l_k(t)$. At any $s' \geq t$, from Lemma~\ref{lem:nested_CIs} we have that $r_i(s') \leq r_i(t) < l_k(t) \leq l_k(s')$, i.e., the arm $k \notin \mathcal{A}_i^0(s') \cup \mathcal{A}_i^1(s')$. 
\end{proof}

\begin{lemma}\label{lem:level3+_dont_matter}
The arms $i_r^t, i_l^t$ are such that $\gaprucb_i(t) = r_{i_r^t}(t) - l_{i_l^t}(t)$. For the sets as defined in \eqref{eq:level0}, \eqref{eq:level1}, \eqref{eq:level2} the arm $i_r^t \in \mathcal{A}_i(t) \triangleq \mathcal{A}_i^0(t) \cup \mathcal{A}_i^1(t) \cup \mathcal{A}_i^2(t)$.
\end{lemma}
\begin{proof}
Suppose arm $i_r^t \notin \mathcal{A}_i(t)$, then either $r_{i_r^t}(t) < l_i(t)$ which would give a negative value for $\gaprucb_i(t)$, or we have that $l_{i_r}(t) > \min_{a: l_a(t) > r_i(t)}r_a(t) \triangleq r_{a_\ast}(t)$. 
From the definition of arm $i_l^t$, its $l_{i_l^t}(t) \leq r_i(t)$. Using this and \eqref{eq:G^R}, we have that
\begin{equation}
G_i^r(l_{i_l^t}(t), t) = \min_{j: l_j(t) > l_{i_l^t}(t)}r_j(t) - l_{i_l^t}(t) 
\leq \min_{j: l_j(t) > r_i(t)} r_j(t) - l_{i_l^t}(t) = r_{a_\ast}(t) - l_{i_l^t}(t).
\end{equation}
From the definition of arm $i_r^t$, we have that $\gaprucb_i(t) = r_{i_r^t}(t) - l_{i_l^t}(t) \leq r_{a_\ast}(t)- l_{i_l^t}(t)$ as argued above. That implies $r_{a_\ast}(t) \geq r_{i_r^t}(t) > l_{i_r^t}(t)$, which contradicts the supposition.
\end{proof}

\begin{lemma}\label{lem:i_R_cant_be_level1_at_>tau'}
At any time $t \geq t(i, 2\tau)$, all arms $j \in \mathcal{A}_i^0(t) \cup \mathcal{A}_i^1(t)$ are such that $r_j(t) - l_i(t) \leq \Delta_{\max}$.
\end{lemma}
\begin{proof}
Consider an arm $j \in \mathcal{A}_i^0(t) \cup \mathcal{A}_i^1(t)$, then $j \notin \mathcal{A}_i^2(s)$ for all $s \leq t$ for otherwise that would contradict corollary~\ref{crlry:level3+_dont_become_2}. 
Thus $j \in \mathcal{A}_i^0(s) \cup \mathcal{A}_i^1(s)$ for all $s \leq t$. 

By choice of $\tau$ we have that $r_i(t(i, \tau)) - l_i(t(i, \tau)) = 2c_{\tau} \leq \Delta_{\max}$ from \eqref{eq:tau defined so that}. Hence if arm $j \in \mathcal{A}_i^0(s)$ for any $s \in [t(i, \tau), t(i, 2\tau)]$, we have that $r_j(s) - l_i(s) \leq r_i(s) - l_i(s) \overset{\text{(a)}}{\leq} r_i(t(i, \tau)) - l_i(t(i, \tau)) \leq \Delta_{\max}$, where inequality (a) is by Lemma~\ref{lem:nested_CIs}. 

Hence $j \in \mathcal{A}_i^1(s)$ for all $s \in [t(i, \tau), t(i, 2\tau)]$. If $\gaprucb_i(s)$ is the largest gap upper bound then $i_R^s \notin \mathcal{A}_i^0(s)$ by the above reasoning. Then Lemma~\ref{lem:all_in_level_sampled} states that arm $j$ was sampled anytime arm $i$ was sampled between $t(i, \tau)$ to $t(i, 2\tau)$. This implies that $T_j(t(i, 2\tau)) \geq \tau$, and we argue that $r_j(t(i, 2\tau)) - l_i(t(i, 2\tau)) \leq \Delta_{\max}$ in the following manner. The arm $j_{\ast}$ is the maximizer in \eqref{eq:gammair definition}.

\textbf{Case I:} $\mu_i < \mu_{j_{\ast}} < \mu_j$. Here we argue that $j \notin \mathcal{A}_i^1(t(i, 2\tau))$ because $l_j(t(i, 2\tau)) \geq r_i(t(i, 2\tau))$ as shown below.
\begin{align*}
l_j(t(i, 2\tau)) - r_i(t(i, 2\tau)) 
&\geq \mu_j - 2c_{T_j(t(i, 2\tau))} - (\mu_i + 2c_{T_i(t(i, 2\tau))}) \quad(\text{Lemma~\ref{lem:simple ci bounds}})\\
&\geq \mu_{j_{\ast}} - \mu_i - 4c_{\tau} \quad(\text{Assumption on means and monotonicity of }c(s))\\
&\geq \mu_{j_{\ast}} - \mu_i - \Delta_{i, j_{\ast}} = 0. \quad(\text{Using \eqref{eq:tau defined so that}})
\end{align*}

\textbf{Case II:} $\max\{\mu_i, \mu_j\} < \mu_{j_{\ast}}$. Here we argue that $r_j(t(i, 2\tau)) - l_i(t(i, 2\tau)) \leq \Delta_{\max}$ as shown below.
\begin{align*}
r_j(t(i, 2\tau)) - l_i(t(i, 2\tau)) 
&\leq \mu_j + 2c_{T_j(t(i, 2\tau))} - (\mu_i - 2c_{T_i(t(i, 2\tau))}) \quad(\text{Lemma~\ref{lem:simple ci bounds}})\\
&\leq \mu_{j_{\ast}} - \mu_i + 4c_{\tau} \quad(\text{Assumption on means and monotonicity of }c(s))\\
&\leq \mu_{j_{\ast}} - \mu_i + \Delta_{\max} - \Delta_{i, j_{\ast}} \leq \Delta_{\max}. \quad(\text{Using \eqref{eq:tau defined so that}})
\end{align*} 
\end{proof}

\begin{lemma}\label{lem:arms_that_can_access_sampled}
Suppose arm $i$ is sampled at time $t$ because $\gaprucb_i(t) = r_{i_r^t}(t) - l_{i_l^t}(t)$ is the largest gap upper bound. Consider an arm $j$ whose confidence bounds satisfy any one of the following conditions.
\begin{enumerate}
\item $l_j(t) < l_{i_l^t}(t) < r_j(t)$, or
\item $l_j(t) < r_{i_r^t}(t) < r_j(t)$.
\end{enumerate}
Then \maxgapucb samples arm $j$ as well at time $t$.
\end{lemma}
\begin{proof}
Suppose arm $j$ satisfies condition (1). Consider the right gap of arm $j$, we have that $\gaprucb_j(t) \geq G_{j}^r(l_{i_l^t}(t), t)$. If the value of $G_i^r(l_{i_l^t}(t), t)$ is obtained by the first branch of \eqref{eq:G^R}, then the value of $G_{j}^r(l_{i_l^t}(t), t)$ is also given by its first branch. That implies $\gaprucb_i(t) = \gaprucb_j(t)$, and hence $j$ is sampled if $i$ is sampled. 
If $j = i_r^t$, by condition (1) we have that $l_{i_r^t}(t) < l_{i_l^t}(t)$, which implies that $G_i^r(l_{i_l^t}(t), t)$ is obtained by the second branch in \eqref{eq:G^R}. Hence for all arms $a \neq i_l^t$ we have $l_a(t) < l_{i_l^t}(t)$ and $r_{i_r^t}(t) = r_j(t) = \max_{a \neq i}r_a(t)$. Considering the left gap of arm $j$, since $\{a: r_a(t) < r_j(t)\} \neq \emptyset$,
\begin{equation*}
G_j^l(r_j(t), t) = r_j(t) - \max_{a: r_a(t) < r_j(t)}l_a(t) = r_{i_r^t}(t) - l_{i_l^t}(t) = \gaprucb_i(t),
\end{equation*}
and arm $j$ is sampled if $i$ is sampled. 
Finally suppose the value of $G_i^r(l_{i_l^t}(t), t)$ is obtained by the second branch in \eqref{eq:G^R}, and $j \neq i_r^t$. Then
\begin{align*}
G_i^r(l_{i_l^t}(t), t) &= \max_{a \neq i}r_a(t) - l_{i_l^t}(t) = r_{i_r^t}(t) - l_{i_l^t}(t),\\
G_{j}^r(l_{i_l^t}(t), t) &= \max_{a \neq j} r_a(t) - l_{i_l^t}(t) = \max\{r_i(t), r_{i_r^t}(t)\} - l_{i_l^t}(t) = r_{i_r^t}(s) - l_{i_l^t}(t),
\end{align*}
where the last equality is true because if not, then $\gaprucb_j(t) \geq G_j^r(l_{i_l^t}(t), t) > G_i^r(l_{i_l^t}(t), t) = \gaprucb_i(t)$, which contradicts the condition that $\gaprucb_i(t)$ is the largest. Hence arm $j$ is sampled if $i$ is sampled.

Suppose now that arm $j$ satisfies condition (2). We divide the proof of this part into two cases. 

\textbf{Case I:} Suppose $r_{i_l^t}(t) > r_{i_r^t}(t)$. 

If the arm $i_l^t \neq i$, then we show that $G_i^r(l_{i_l^t}(t), t)$ cannot be the largest gap upper bound. 
Consider the arm $a_\ast \triangleq \arg\max_{a: l_a(t) < l_{i_l}(t)} l_a(t)$, it satisfies $l_i(t) \leq l_{a_\ast}(t) < l_{i_l^t}(t)$. Then $G_i^r(l_{a_\ast}(t), t) = \min_{a: l_a(t) > l_{a_\ast}(t)} r_a(t) - l_{a_\ast}(t)$, where the first branch of \eqref{eq:G^R} is active because of arm $i_l^t$. But
\begin{equation*}
\min_{a: l_a(t) > l_{a_\ast}(t)} r_a(t) = \min\{r_{i_l^t}(t), \min_{a: l_a(t) > l_{i_l^t}(t)} r_a(t)\} = \min\{r_{i_l^t}(t), r_{i_r^t}(t)\} = r_{i_r^t}(t).
\end{equation*}
That would imply 
\begin{equation*}
G_i^r(l_{a_\ast}(t), t) = r_{i_r^t}(t) - l_{a_\ast}(t) > r_{i_r^t}(t) - l_{i_r^t}(t) = G_i^r(l_{i_r^t}(t), t),
\end{equation*}
which contradicts the identification of arm $i_l^t$ as the one giving the value of $\gaprucb_i(t)$. 
The case that remains is if the arm $i = i_l^t$. For this part consider the following two sub-cases: 

\quad \textbf{Sub-case Ia:} The set of arms $\{a: r_a(t) < r_{i_r^t}(t)\} = \emptyset$. 
Since the number of arms $K > 2$, the value $\max_{a \neq i}r_a(t) > r_{i_r^t}(t)$, and hence if $G_i^r(l_{i_l^t}(t), t) = r_{i_r^t}(t) - l_{i_l^t}(t)$, then it must be due to the first branch in \eqref{eq:G^R}. That implies $l_{i_r^t}(t) > l_{i_l^t}(t) = l_i(t)$. Then consider the left gap for arm $j$ that satisfies condition (2). Since the set $\{a: r_a(t) < r_{i_r^t}(t)\} = \emptyset$, we have
\begin{equation*}
G_j^l(r_{i_r^t}(t), t) = r_{i_r^t}(t) - \min_{a \neq j}l_a(t) = r_{i_r^t}(t) - l_{i_l^t}(t) = \gaprucb_i(t),
\end{equation*}
which implies that arm $j$ will be sampled if $\gaprucb_i(t)$ is the largest.

\quad \textbf{Sub-case Ib:} The set of arms $\{a: r_a(t) < r_{i_r^t}(t)\} \neq \emptyset$. 
Consider the arm $a_\ast \triangleq \arg\max_{a: l_a(t) < l_i(t)} l_a(t)$, the domain in the maximization is not empty because of the following. By the case assumption, there is an arm $a$ whose $r_a(t) < r_{i_r^t}(t)$. If the left bound of this arm $l_a(t) > l_i(t)$, then $l_i(t) < l_a(t) < r_a(t) < r_{i_r^t}(t)$, which contradicts the identification of arm $i_r^t$ for $G_i^r(l_i(t), t)$. Hence its left bound must satisfy $l_a(t) < l_i(t)$. Now consider the right gap of arm $a_\ast$ defined above. Since $l_i(t) > l_{a_\ast}(t)$, we have that $G_{a_\ast}^r(l_{a_\ast}(t), t) = \min_{a: l_a(t) > l_{a_\ast}(t)}r_a(t) -  l_{a_\ast}(t)$. But
\begin{equation*}
\min_{a: l_a(t) > l_{a_\ast}(t)}r_a(t) = \min\{r_{i}(t), \min_{a: l_a(t) > l_i(t)}r_a(t)\} 
= \min\{r_i(t), r_{i_r^t}(t)\} = r_{i_r^t}(t),
\end{equation*}
which implies that $G_{a_\ast}^r(l_{a_\ast}(t), t) = r_{i_r^t}(t) - l_{a_\ast}(t) > r_{i_r^t}(t) - l_i(t) = \gaprucb_i(t)$, which is a contradiction.
We are left with the following Case II.

\textbf{Case II:} Suppose $r_{i_l^t}(t) < r_{i_r^t}(t)$.

Let $a_\ast$ be such that $l_{a_\ast}(t) \triangleq \max_{a: r_a(t) < r_{i_r^t}(t)}l_a(t)$. Then $l_{a_\ast}(t) \geq l_{i_l^t}(t)$. If the previous inequality is strict, then we have that
\begin{equation*}
l_{i_l^t}(t) < l_{a_\ast}(t) < r_{a_\ast}(t) < r_{i_r^t}(t),
\end{equation*}
which contradicts the identification of arm $i_r^t$ as the one giving the value of $G_i^r(l_{i_l^t}(t), t)$. Hence we have that
\begin{equation*}
G_j^l(r_{i_r^t}(t), t) = r_{i_r^t}(t) - \max_{a: r_a(t) < r_{i_r^t}(t)}l_a(t)
= r_{i_r^t}(t) - l_{i_l^t}(t) = \gaprucb_i(t),
\end{equation*}
and arm $j$ is sampled if arm $i$ is sampled because of $\gaprucb_i(t)$.
%
%
%
%
%
%
%
\end{proof}

\begin{lemma}\label{lem:all_in_level_sampled}
Suppose arm $i$ is sampled at time $t$ when $\gaprucb_i(t) = r_{i_r^t}(t) - l_{i_l^t}(t)$. If $i_r^t \in \mathcal{A}_i^2(t)$ then all arms in the set $\mathcal{A}_i^1(t) \cup \mathcal{A}_i^2(t)$ are sampled by \texttt{MaxGapUCB}. If $i_r^t \in \mathcal{A}_i^1(t)$ then all arms in the set $\mathcal{A}_i^1(t)$ are sampled by \texttt{MaxGapUCB}.
\todos{break into 2 lemmas? one for level 1 and one for level 2}
\end{lemma}
\begin{proof}
The qualifying condition states that the arm $i_r^t \in \mathcal{A}_i^1(t) \cup \mathcal{A}_i^2(t)$, hence from definitions \eqref{eq:level1}, \eqref{eq:level2} we have that $r_{i_r^t}(t) \geq r_i(t)$. 
By definition \eqref{eq:right_gap_UB} the arm $i_l^t$ is such that $l_{i_l^t}(t) \in [l_i(t), r_i(t)]$. We first argue that all arms in the set $\mathcal{A}_i^1(t)$ are sampled. For arm $j \in \mathcal{A}_i^1(t), r_i(t) \leq r_j(t)$. 
If $l_j(t) \leq l_{i_l^t}(t)$, arm $j$ satisfies condition (1) of Lemma~\ref{lem:arms_that_can_access_sampled} and hence is sampled if $\gaprucb_i(t)$ is the largest. 
If on the other hand $r_j(t) \geq r_{i_r^t}(t)$, then arm $j$ satisfies condition (2) of Lemma~\ref{lem:arms_that_can_access_sampled} and hence it is sampled if $i$ is sampled. 
The remaining case is if $l_{i_l^t}(t) < l_j(t) < r_j(t) < r_{i_r^t}(t)$, but that would contradict the identification of the arm $i_r^t$ for $\gaprucb_i(t)$.

Now suppose arm $i_r^t \in \mathcal{A}_i^2(t)$, what is left to prove is that all arms in the set $\mathcal{A}_i^2(t)$ are sampled. Since $i_r^t \in \{a: l_a(t) > r_i(t) > l_{i_l^t}(t)\}$, we have that
\begin{equation*}
G_i^r(l_{i_l^t}(t), t)= \min_{j:l_j(t) > l_{i_l^t}(t)}r_j(t) - l_{i_l^t}(t) = r_{i_r^t}(t) - l_{i_l^t}(t) = \min_{j \in \mathcal{A}_i^2(t)}r_j(t) - l_{i_l^t}(t),
\end{equation*}
where the last equality is true because arm $i_r^t \in \mathcal{A}_i^2(t)$ satisfies $l_{i_r^t}(t) > r_i(t) \geq l_{i_l^t}(t)$. 
From definition \eqref{eq:level2}, any $j \in \mathcal{A}_i^2(t)$ is such that $l_j(t) \leq r_{i_r^t}(t)$ and satisfies condition (2) of Lemma~\ref{lem:arms_that_can_access_sampled}. 
Hence arm $j$ is sampled if $i$ is sampled because of its right gap.
\end{proof}

\section{Details for \cref{sec:lowerbound}: Proof of \cref{lem:lower bound lemma}}
\label{sec:appendix lower bound}
\Minimax*
\begin{proof}  
The maximum gap in $\mathcal{B}$ is $\Delta_{\max} = \Delta_{3,2} = \nu + \epsilon$. 
Define an alternate bandit model $\mathcal{B}'$ with $4$ normal distributions
$\mathcal{P}'_i = \mathcal{N}(\mu_i', 1)$ where
\begin{equation*}
\mu_i' = \mu_i \quad \forall i \neq 4,\qquad \mu_4' = 2.1\epsilon.
\end{equation*}
\begin{figure}[h]
\centering
\includegraphics[width=.5\linewidth]{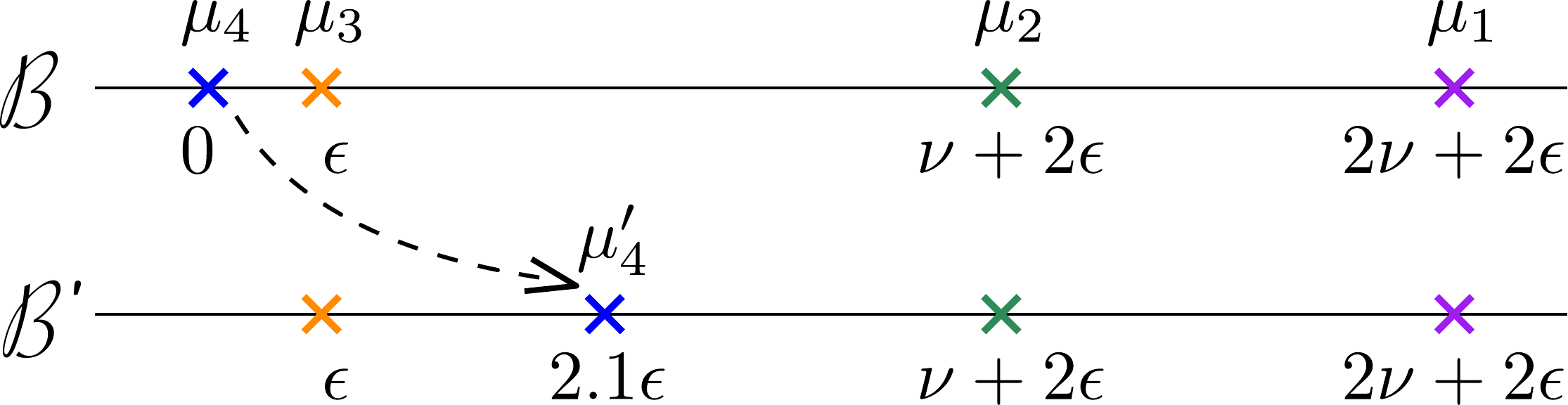} 
\caption{Changing the original bandit model $\mathcal{B}$ to
  $\mathcal{B}'$. $\mu_4$ is
  shifted to the right by $2.1\epsilon$. As a result, the maximum gap
  in $\mathcal{B}'$ is between green and purple.  }
\label{fig:appendix change_of_measure}
\end{figure}
Note that the ordering of the means in $\mathcal{B}'$ does not follow the
subscript indices, indeed $\mu_3' < \mu_4'$. The two measures 
are illustrated in \cref{fig:appendix change_of_measure}.
The maximum gap in $\mathcal{B}'$ is $\Delta_{\max}' = \Delta_{2, 1}' = \nu$ and $\Delta_{3, 2}'$ is no longer a valid gap between consecutive arms. 
Consider algorithm for identifying the maximum gap and let
$\widehat{C}_1$ denote the top-cluster returned by the algorithm when it stops at
time $\tau$. Let $E = \{\widehat{C}_1 = \{1,2\}\}$. Assume that $\mathbb{P}_{\mathcal{B}}(E) \geq
1-\delta$ and $\mathbb{P}_{\mathcal{B}'}(E) \leq \delta$.
Letting $d(\cdot)$ denote the binary relative entropy, Lemma~1 in
\citet{garivier2016optimal} implies that
\begin{align*}
\sum_{a=1}^4 \mathbb{E}_{\mathcal{B}}[T_a(\tau)]\mathsf{KL}(\mathcal{P}_a, \mathcal{P}_a') &\geq d(\mathbb{P}_{\mathcal{B}}(E), \mathbb{P}_{\mathcal{B}'}(E)) \geq d(1-\delta, \delta)\\
\implies \mathbb{E}_{\mathcal{B}}[T_4(\tau)] (\mu_4 - \mu_4')^2 &\geq
\log\frac{1}{2.4\delta} \implies \mathbb{E}_{\mathcal{B}}[T_4(\tau)] \geq \frac{1}{(2.1\epsilon)^2}\log\frac{1}{2.4\delta}.
\end{align*}
Similarly, one can show that $\mathbb{E}_{\mathcal{B}}[T_1(\tau)] \ge
1/\epsilon^2$ by creating an alternative bandit instance $\mathcal{B}''$
identical to $\mathcal{B}$ 
except $\mu_1^{''} = 2\nu + 3.1\epsilon$. 
\end{proof}